\documentclass{tlp}
\usepackage{enumitem}

\usepackage{amsmath}
\usepackage{graphicx}
\usepackage{multirow}

\usepackage{mathptmx}

\usepackage{multicol}
\usepackage{times}
\usepackage{soul}
\usepackage{url}
\usepackage[hidelinks]{hyperref}
\usepackage[utf8]{inputenc}
\usepackage[small]{caption}
\usepackage{graphicx}
\usepackage{booktabs}
\usepackage{algorithm}
\usepackage{algorithmic}
\def\nuz{$\nu${Z}\xspace}

\newtheorem{example}{Example}
\newtheorem{theorem}{Theorem}

\usepackage{times}
\usepackage{helvet}
\usepackage{courier}
\usepackage{xspace}
\usepackage{latexsym}
\usepackage{verbatim}
\usepackage{amsmath}
\usepackage{amssymb}
\usepackage{color}
\usepackage{tikz}
\usetikzlibrary{shapes,arrows}
\usepackage{hyperref}

\usepackage{booktabs}
\usepackage{alltt}
\usepackage{caption}
\def\wr{:\sim}

\DeclareMathOperator*{\argmax}{arg\,max}

\def \br#1{[{#1}]}
\def \signp#1{{#1}^+}
\def \signm#1{{#1}^-}
\def \signo#1{{#1}^{-1\cdot}}
\def \signoo#1{{#1}^{-1\cdot-1\cdot}}

\def\level#1{{\lambda({#1})}}

\newcommand{\spn}[2]{[{#1}, {#2}]}
\newcommand{\pair}[2]{({#1}, {#2})}
\def \less#1#2{{#1}[{\setminus{#2}}]}

\newtheorem{proposition}[theorem]{Proposition}

\newtheorem{definition}{Definition}

\renewcommand{\vec}[1]{{\bf #1}}

\def\citeb#1{(\citeauthor{#1}, \citeyear{#1})}
\def\citebb#1#2{(\citeauthor{#1}, \citeyear{#1}; \citeauthor{#2}, \citeyear{#2})}
\def\shortcite#1{(\citeyear{#1})}

\def\<{\langle}
\def\>{\rangle}

\def \prec#1{{#1}^\uparrow}

\def\V{\Upsilon}
\def\D{\Delta}

\def \signpp#1{{#1}^{+;+}}
\def \signmm#1{{#1}^{-;-}}

\def\Int{\ensuremath{\mathit{Int}}}

 \def\cS{\ensuremath{\mathcal{Z}}}
 \def\cF{\ensuremath{\mathcal{F}}}
 
 \def\cP{\ensuremath{\mathcal{P}}}
 \def\cF{\ensuremath{\mathcal{F}}}
 \def\cW{\ensuremath{\mathcal{W}}}
 
 \def\cL{\ensuremath{\mathcal{L}}}

 \def\cH{\ensuremath{\mathcal{H}}}

\newcommand{\C}{\mathcal{C}}

\newcommand{\ignore}[1]{}

\def\sem{\mathit{sem}}

\def\ar{\leftarrow}
\def\rar{\rightarrow}

 \def\beq{\begin{equation}}
 \def\eeq#1{\label{#1}\end{equation}}
 \def\ba{\begin{array}}
 \def\ea{\end{array}}

\def\clasp{{\sc clasp}\xspace}

\newtheorem{property}{Property}

\newcommand{\At}{\mathit{At}}

\newcommand{\hd}{\mathit{hd}}

\begin{document}

\lefttitle{Yuliya Lierler}

  \title[Unifying Framework for Optimizations]{Unifying Framework for Optimizations in non-boolean Formalisms}
  
  \begin{authgrp}
\author{\gn{Yuliya Lierler}}
\affiliation{University of Nebraska Omaha}
\end{authgrp}

\jnlPage{\pageref{firstpage}}{\pageref{lastpage}}
\jnlDoiYr{2021}
\doival{10.1017/xxxxx}

\maketitle

\begin{abstract}
Search-optimization problems are plentiful in scientific and engineering domains. Artificial intelligence has long contributed to the development of search algorithms and declarative programming languages geared towards solving and modeling search-optimization problems.
Automated reasoning and knowledge representation are the subfields of AI that are particularly vested in these developments.
Many popular automated reasoning paradigms provide users with languages supporting optimization statements. Recall  integer linear programming, MaxSAT, optimization satisfiability modulo theory, (constraint) answer set programming.
These paradigms vary significantly in their languages in ways they express quality conditions on computed solutions. Here we propose a unifying framework of so called extended weight systems that eliminates syntactic distinctions between paradigms. They allow us to see essential similarities and differences between optimization statements provided by distinct automated reasoning languages. We also study formal properties of the proposed systems that immediately translate into formal properties of paradigms that can be captured within our framework. Under consideration in Theory and Practice of Logic Programming (TPLP).
    \end{abstract}




\label{firstpage}

\section{Introduction}
Artificial intelligence is a powerhouse for delivering algorithmic frameworks to support solutions to
search-optimization problems that are plentiful in scientific and engineering domains. 
Automated reasoning and knowledge representation are the subfields of AI that are particularly vested in developing general-purpose search algorithms and declarative programming languages specifically geared towards formulating constraints of search-optimization problems.
Various automated reasoning paradigms provide users with languages supporting optimization statements. Indeed, consider  such popular paradigms as integer linear programming (ILP)~\citeb{pap82}, MaxSAT~\citeb{rob10}, optimization satisfiability modulo theory (OMT)~\citebb{nie06a}{seb12}, answer set programming with weak constraints~\citeb{alv18}, constraint answer set programming (CASP)~\citeb{ban17}.
These paradigms allow a user to express  ``hard'' and ``soft'' constraints given a problem of interest.
Hard part of an encoding for a considered problem is meant to state requirements on what constitutes a solution to this problem. Soft part of the encoding is meant to state optimization criteria based on which we compare resulting solutions and find  optimal ones. For example, integer linear programs have the form
\beq
\ba{ll}
\hbox{{\tt maximize}} & \hbox{{ $\vec{c}^{\vec{T}} \vec{x}$ }}\\
\hbox{{\tt subject to }} & \hbox{{$\vec{A}\vec{x}\leq \vec{b}$ ; ~~}} \vec{x}\geq 0
\hbox{{\tt; and }} \\
&\hbox{{$\vec{x}\in \mathbb{Z}^n$ }}\\
\ea
\eeq{eq:ilp}
where $\vec{c}$, $\vec{b}$ are vectors and $\vec{A}$ is a matrix whose all entries are integers. The {\tt maximize} statement encodes soft constraints, whereas the {\tt subject to} statements encode hard constraints. On the other hand, in partially weighted MaxSAT the statements to encode both hard and soft part have the form of a clause~~
$$
l_1\vee\dots\vee l_n,
$$
where $l_i$ is a literal (recall that a literal is either an atom or its negation, where an atom is a binary/propositional variable). Clauses of the soft part are associated with weights and the goal is to maximize the sum of weights for clauses satisfied by a model of the hard part.   These samples of two distinct automated reasoning paradigms offering optimizations point at clear differences. These paradigms vary significantly in their languages, for example, in ways they express the ``hard constraints'', in ways they express  ``soft constraints'', as well as their vocabularies. Indeed, ILP primarily objects are variables over integers, whereas in MaxSAT we are looking at binary variables. Furthermore, in formalisms such as optimizations modulo theory or constraints answer set programming both kinds of variables are allowed {\em together with intricate interface between these}.  

The  relations between many of the enumerated paradigms have been studied in the absence of ``soft constraints.'' Recently,~\cite{lie21,lie22} provided a formal account comparing MaxSAT/MinSAT family~\citeb{rob10} and answer set programs with weak constraints~\citeb{alv18} (weak constraints are syntactic objects to express soft constraints in logic programs).
In that work,
to draw the precise parallel between these frameworks so called abstract modular weight-systems (w-systems) served the role of a 
 primary tool. These systems abstracted away syntactic differences of the paradigms, while leaving the essential semantic ingredients sufficient to introduce a concept of a model and optimization criteria for their comparison. An abstract notion of a logic introduced by \cite{bre07} is a crucial ingredient of w-systems. This abstract logic may encapsulate languages over binary variables. Hence, w-systems may capture frameworks such as MaxSAT or logic programs with optimizations.

In this work, we extend the concepts of an abstract logic and w-systems to provide us with a framework capable to capture formalisms that utilize distinct kinds of variables in their languages. Then, we show how resulting extended w-systems encapsulate ILP, OMT (in its two common variants), and CASP with optimization statements. We trust that such birds eye view on these distinct paradigms  will boost cross-fertilization between approaches used in design of algorithms supporting optimizations in distinct fields.
Indeed,~\cite{lie21,lie22} illustrated how MaxSAT/MinSAT solvers can be used to compute optimal answer sets for logic programs with weak constraints by providing theoretical basis for cross-translations between formalisms. This work provides theoretical grounds for devising translations for related optimization statements in CASP and OMT. Thus, we pave the way, for example, for an extension of a constraint answer set solver {\sc ezsmt}~\citeb{shen18a}. This solver  relies on utilizing   satisfiability modulo theory 
solvers for finding answer sets of a considered CASP program. In the future, this system can utilize OMT solvers to find optimal answer sets of a CASP program.

We would like to point at work by \cite{alv18b}, where the authors also realized the importance of abstracting away the syntactic details of the formalisms to describing hard and soft constraints of a considered problem in order to streamline the utilization of existing efficient solving techniques in new settings. 
\cite{alv18b} define optimization problems  at a semantic level to present translations of several preference relations into minimize/maximize and subset/supset statements, as implemented in the ASPRIN system developed by~\cite{bre15a} and based on answer set programming technology.

\smallskip
\noindent
{\em Paper outline~~} 
We start the paper by reviewing the concepts of an abstract logic and logic programs. We then define a notion of an extended logic. In Section~\ref{sec:elcassmt}, we show how extended logics naturally capture constraint answer set programs and satisfiability modulo theory formulas (reviewed in the same section). Section~\ref{sec:ewsystems} introduces the central concept of this work: extended weighted modular systems. Then, we use these modular systems to capture a variety of automated reasoning paradigms with optimization statements, namely, several OMT and {\sc clingcon}-based frameworks. In addition, we provide natural generalizations to these frameworks utilizing introduced modular systems. We conclude by enumerating formal properties of these systems and an account of proofs for presented formal results.

\section{Review: Abstract Logic; Logic Programs}
\label{mbms}

A \emph{language} is a set $L$ of \emph{formulas}. A \emph{theory} is 
a subset of~$L$. Thus the set of theories is closed under union and
has the least and the greatest elements: $\emptyset$ and $L$.
This definition ignores any syntactic details behind the concepts of
 a formula and a theory.
A \emph{vocabulary} is an infinite countable set of \emph{atoms}.
Subsets of a vocabulary $\sigma$ represent (classical propositional)
\emph{interpretations} of $\sigma$. We write $\Int(\sigma)$ for the 
family of all interpretations of a vocabulary $\sigma$. 
%

\begin{definition}
A \emph{logic} is a triple $\cL=(L_\cL,\sigma_\cL,\sem_\cL)$, where
\begin{enumerate}
\item $L_\cL$ is a language (the language of the logic $\cL$),
\item $\sigma_\cL$ is a vocabulary (the vocabulary of the logic $\cL$),
\item $\sem_\cL:2^{L_\cL} \rightarrow 2^{\Int(\sigma_\cL)}$ is 
a function assigning collections of interpretations to theories in $L_\cL$
(the \emph{semantics} of $\cL$).
\end{enumerate}
\end{definition}
If a logic $\cL$ is clear from the context, we omit the subscript $\cL$ 
from the notation of the language, the vocabulary and the semantics of 
the logic. 

\emph{Literals} over a vocabulary $\sigma$ are expressions $a$ and $\neg a$, 
where $a$ is an atom from $\sigma$.
A \emph{(logic) rule} over $\sigma$ is of the form 
\begin{equation}\label{e:rule}
\begin{array}{l}
a_0\ar a_1,\dotsc, a_\ell,\ not\  a_{\ell+1},\dotsc,\ not\  a_m,\ 
\end{array}
\end{equation}
where $a_0$ is an atom in $\sigma$ or $\bot$ (empty), and each $a_i$, $1\leq i\leq m$, 
is an atom in $\sigma$.
A \emph{logic program}  over $\sigma$ is a 
set of \emph{rules} over $\sigma$.
The expression $a_0$ is the \emph{head} of the rule. 
The expression on the right hand side of the arrow is the \emph{body}.
We write $\hd(\Pi)$ for the set of nonempty heads of
rules in logic program~$\Pi$.
It is customary for a given vocabulary $\sigma$, to identify a set~$X$ of atoms over $\sigma$ with (i) a complete and consistent set of literals over $\sigma$ constructed as $X\cup\{\neg a \mid a\in\sigma\setminus X\}$, and respectively with (ii)~an assignment function  that assigns the truth value {\em true} to every atom in~$X$ and~{\em false} to every atom in $\sigma\setminus X$. In the sequel, we may refer to sets of atoms as assignments and the other way around following this convention. 
We say that a set~$X$ of atoms {\em satisfies} rule~\eqref{e:rule}, if~$X$ satisfies the propositional formula
$$
a_1\wedge\dotsc\wedge a_\ell\wedge\ \neg  a_{\ell+1}\wedge\dotsc\wedge\ \neg  a_m\rar a_0.\ 
$$
The {\sl reduct}~$\Pi^X$ of a program~$\Pi$ relative to a set~$X$ of atoms is 
obtained by first removing all rules~\eqref{e:rule} such that~$X$ does not satisfy the propositional formula corresponding to the negative part of the body
$\neg  a_{\ell+1} \wedge\ldots\wedge\neg  a_m ,$
and replacing all remaining rules with~~
$$a\ar a_1,\ldots, a_\ell.$$
A set~$X$ of atoms is an {\em answer set}, if it is the minimal set that satisfies all rules of $\Pi^X$~\citeb{lif99d}. 
For example, program
\beq
\ba{l}
a\ar not\ b\\
b\ar not\ a.
\ea
\eeq{ex:slp}
has two answer sets $\{a\}$ and $\{b\}$.

A set~$X$ of atoms from a vocabulary $\sigma$ is an \emph{input answer set} 
of a logic program $\Pi$ over $\sigma$ if $X$ is an answer set of the program 
$\Pi\cup (X\setminus \hd(\Pi))$~\citeb{lt2011}.
For example, if we consider program~\eqref{ex:slp} as a program over vocabulary $\{a,b,c\}$, then program~\eqref{ex:slp} has four input answer sets:  $\{a\}$, $\{b\}$, $\{a,c\}$, $\{b,c\}$.

Brewka and Eiter~\shortcite{bre07} showed that their abstract notion 
of a logic captures default logic, propositional logic, and logic programs 
under the answer set semantics. For example, the logic  $\cL=(L,\sigma,\sem)$,
where
\begin{enumerate}
\item $L$ is the set of propositional formulas over $\sigma$,
\item $\sem(F)$, for a theory $F\subseteq L$, is the set of propositional
models of $F$ over $\sigma$,
\end{enumerate}
captures propositional logic.
We call this logic $\cL$ the {\em pl-logic} and theories (that we later call modules) in the pl-logic, 
\emph{pl-theories/modules}.
If we restrict elements of $L$ to be  clauses, then we call $\cL$  a {\em sat-logic}. 

Similarly, 
a logic  $\cL=(L,\sigma,\sem)$, where
\begin{enumerate}
\item $L$ is the set of logic program rules over $\sigma$,
\item $\sem(\Pi)$, for a program $\Pi\subseteq L$, is the set of answer sets/input answer sets of $\Pi$ over $\sigma$,
\end{enumerate}
captures logic programs under the answer set/input answer set semantics.
We call these logics the  {\em lp-logic} and {\em ilp-logic} respectively and theories/modules in these logics,
\emph{lp-theories/modules} and \emph{ilp-theories/modules}.

\section{Extended Logic for CAS Programs and SMT Formulas}\label{sec:elcassmt}


\paragraph{Review: Constraint Satisfaction Problems and Constraint Answer Set Programs}
A pair~$\spn{V}{D}$, where $V$ is a set of variables and $D$ is a set of values for variables in
$V$ or the \emph{domain} for $V$, is called a {\em specification}. A \emph{constraint} over specification $\spn{V}{D}$ is a 
pair $\langle t,\mathbb{R} \rangle$, where $t$ is a tuple of some (possibly all)
variables from $V$ and $\mathbb{R}$ is a relation on $D$ of the same arity as 
$t$. A collection of constraints over $\spn{V}{D}$ is a \emph{constraint 
satisfaction problem} (CSP) over $\spn{V}{D}$. An {\em evaluation} of
$V$ is a function assigning to every variable in $V$ a value from $D$. 
An evaluation~$\nu$ {\em satisfies} a constraint $\langle (x_1,\ldots,
x_n), \mathbb{R} \rangle$ (or is a {\em solution} of this constraint) if 
$(\nu(x_1),\ldots, \nu(x_n)) \in \mathbb{R}$. An evaluation \emph{satisfies} (or
is a \emph{solution} to) a constraint satisfaction problem if it satisfies 
every constraint of the problem. 
Let $c=\<t,\mathbb{R}\>$ be a constraint and $D$ the domain of its
variables. Let $k$ denote the arity of $t$.
The constraint $\overline{c}=\<t,D^k\setminus \mathbb{R}\>$ is the \emph{complement}
(or \emph{dual}) of $c$. Clearly, an evaluation of variables in~$t$
satisfies $c$ if and only if it does not satisfy $\overline{c}$. 
Frequently, constraints are stated implicitly without a reference to explicit relation. In particular, constraints over integers or reals are often formulated by means of common arithmetic relations.
\begin{example}\label{ex:0}
A constraint $c_1=\langle x, \langle (1), (2) \rangle \rangle$ over specification $s_1=\spn{\{x\}}{\{0,1,2\}}$ can be implicitly represented as inequality $x\geq 1$ or inequality $x\neq 0$.
\end{example}

When we refer to some {\em class of constraints} we assume that all of its members are over the same specification.
We call a function from a vocabulary $\sigma$ to a class $\C$ of constraints a {\em $\pair{\sigma}{\C}$-denotation}.
From now on given a vocabulary $\sigma$ we assume that some of its atoms are marked as {\em irregular/constraint} atoms.
We  utilize subscripts $r$ and $c$ to refer to {\em regular} atoms -- atoms not marked as constraint ones -- and {\em constraint} atoms, respectively. Thus, given vocabulary~$\sigma$: 
$\sigma_r$ forms the subset of $\sigma$  containing all regular atoms and 
    $\sigma_c$ forms the subset of $\sigma$ containing all constraint atoms.

We present the   definition of answer sets for CAS programs as proposed by
Lierler~\shortcite{lier14} using the notation of this paper.
\begin{definition}[CAS Programs and their Answer Sets]
A {\em CAS rule} over vocabulary $\sigma$ is a logic rule~\eqref{e:rule} over $\sigma$,
where $a_0$ is an atom in $\sigma_r$ or $\bot$ (empty), and each $a_i$, $1\leq i\leq m$, is an atom in~$\sigma$.
A {\em  constraint answer set 
program (CAS program)} $P$ over vocabulary $\sigma$, class $\C$ of constraints, and $\pair{\sigma_c}{\C}$-denotation $\gamma$  is a set of
CAS rules over $\sigma$.
Given a CAS program (or logic program or propositional formula) $P$, by $\At(P)$ we denote atoms occurring in $P$.
A set $X\subseteq\At(P)$ 
is an \emph{answer set}
of $P$ if
\begin{itemize}
\item  $X\cap\At(P)_r\subseteq\hd(P)$, 
\item  $X$ is an input answer set of $P$, and 
\item the following CSP has a solution\footnote{Gebser et al.~\shortcite{geb16} proposed a definition for CAS programs that assumes two kinds of irregular/constraint atoms: strict-irregular  and non-strict-irregular. In our presentation here constraint atoms behave as strict-irregular atoms. Changing condition~\eqref{eq:caspcond} to~$\{\gamma(a)\colon a\in X\cap\At(\Pi)_c\}$
 turns the behavior of constraint atoms to non-strict-irregular atoms. 
}
\beq
\ba{l}
\{\gamma(a)\mid a\in X\cap\At(P)_c\}\cup 
\{\overline{\gamma(a)}\mid a\in \At(P)_c\setminus X\}.
\ea
\eeq{eq:caspcond}
\end{itemize}
A pair $(X,\nu)$ 
is an \emph{\underline{extended} answer set} of CAS program $P$ if $X$ is an answer set of $P$ and $\nu$ is a solution to~\eqref{eq:caspcond}.
\end{definition}

\begin{example}\label{ex:casp}
Consider specification $s_1$ and a class $\C_1$ of constraints consisting of constraint $c_1$ and its complement, where  $s_1$ and  $c_1$ are defined in Example~\ref{ex:0}.
Take vocabulary $\sigma_1$ contain two regular atoms $a,b$ and an irregular atom $|x\neq 0|$; and assume
$({\sigma_1}_c,\C_1)$-denotation that maps irregular atom $|x\neq 0|$ to constraint $c_1$. Let $P_1$ be CAS
program over $\sigma_1$, $\C_1$, and $({\sigma_1}_c,\C_1)$-denotation
consisting of rules in~\eqref{ex:slp} and rule~~
$$\ar a, |x\neq 0|.$$
Take $\nu_0$, $\nu_1$, $\nu_2$ be evaluations assigning $x$ to $0$, $1$, and $2$, respectively.
The extended answer sets of $P_1$ are~~ $(\{a\},\nu_0)$,~~ $(\{b\},\nu_0)$,~~ $(\{b,|x\neq 0|\},\nu_1)$,~~ $(\{b,|x\neq 0|\},\nu_2)$.
\end{example}



Lierler and Truszczynski~\shortcite{lt2015} introduced abstract modular systems based of conglomerations of theories over various logics. They illustrated that such systems can be used to capture CAS programs: in particular, they may capture  answer sets of a given CAS program. Here, we introduce an extended notion of a logic so that we may speak of abstract systems capturing the meaning of CAS progams in terms of \underline{extended} answer sets. We then show that the concept of extended logic is helpful in capturing problems expressed as satisfiability modulo theories (SMT).

\paragraph{Extended Logic}
For a vocabulary $\sigma$, set $V$ of variables, and domain $D$, we call a pair $(I,\nu)$, where~$I$ 
is an interpretation over $\sigma$ and $\nu$ is an evaluation from $V$ to $D$, {\em an extended interpretation} over  $\sigma$, $V$, $D$.
 We write $\Int(\sigma,V,D)$ for the 
family of all extended interpretations over  $\sigma$, $V$, $D$. 

\begin{definition}[Extended Logic]
An \emph{extended logic or e-logic} $\cL+$ is a tuple $$(L_{\cL+},\sigma_{\cL+},\D_{\cL+}, \V_{\cL+},\sem_{\cL+}),$$ where
\begin{enumerate}
\item $L_{\cL+}$ is a language 
\item $\sigma_{\cL+}$ is a vocabulary 
\item $\D_{\cL+}$ is a domain -- a set of values -- (the domain of the logic $\cL+$)
\item $\V_{\cL+}$ is a a set of variables over domain $\D_{\cL+}$  (the variables of the logic $\cL+$)
\item $\sem_{\cL+}$: is 
a function assigning collections of extended interpretations $(I,\nu)$ to theories in $L_{\cL+}$, where $(I,\nu)$ is an element in  $\Int(\sigma_{\cL+},\V_{\cL+},\D_{\cL+})$. 
\end{enumerate}

In the sequel, we will default to naming the members of the tuples of  extended logic $\cL+$ as in this definition.
\end{definition}
It is easy to see that an extended logic generalizes the concept of a logic: we can identify any logic $\cL=(L_{\cL},\sigma_{\cL},\sem_{\cL})$ with its extended counterpart
$\cL+=(L_{\cL},\sigma_{\cL},\emptyset, \emptyset,\sem_{\cL})$, where  
$\sem_{\cL}$ in $\cL+$ is identified with a function that maps theories into pairs whose first element is an interpretation of $\sem_{\cL}$ in $\cL$ application and the second element is empty function.

We now illustrate that 
extended logic captures CAS programs under extended answer set semantics. Then, we show how SMT formulas are captured by this formalism.
Indeed, 
provided  class $\C$ of constraints over specification $\spn{V}{D}$ and a $\pair{\sigma_c}{\C}$-denotation,  
the extended logic~$(L,\sigma,D,V,\sem)$,
where
\begin{enumerate}
\item $L$ is the set of CAS rules over vocabulary $\sigma$;
\item $\sem(P)$, for a theory $P\subseteq L$, 
 over class $\C$ of constraints and $\pair{\sigma_c}{\C}$-denotation,  
is the set of extended answer sets of $P$,
\end{enumerate}
captures CAS programs under extended answer set semantics.
We call this logic CAS logic.

\paragraph{Satisfiability Modulo Theories as  Theories in Extended Logic}
Here we state the definition of a Satisfiability Modulo Theories (or SMT) formula~\citeb{BarTin-14} following the lines by \cite{lie17} using terminology introduced here. An alternative name for SMT formulas could have been constraint formulas. 

\begin{definition}[SMT formulas and their Models]
An {\em SMT formula} $\cF$ over vocabulary $\sigma$,   class~$\C$ of constraints and $\pair{\sigma_c}{\C}$-denotation $\gamma$ is a set of propositional formulas (often assumed to be clauses) over $\sigma$. A set $X\subseteq \sigma$ of atoms is 
 a \emph{model}
of~$\cF$, denoted  $X\models\cF$,  if
\begin{itemize}
\item  $X$ is a model of propositional classical logic theory $\cF$, and 
\item the CSP~\eqref{eq:caspcond} with $P$ replaced by $\cF$ has a solution.
\end{itemize}
%
A pair $(X,\nu)$ 
is an \emph{\underline{extended} model}, denoted $(X,\nu)\models \cF$,  if $X$ is model of $\cF$ and $\nu$ is a solution to~\eqref{eq:caspcond} with $P$ replaced by $\cF$.
\end{definition}

SMT formulas  
can be captured by extended logic theories just as CAS programs.  
Provided  class $\C$ of constraints over specification $\spn{V}{D}$ and a $\pair{\sigma_c}{\C}$-denotation,  
the extended logic~$(L,\sigma,D,V,\sem)$,
where
\begin{enumerate}
\item $L$ is the set of propositional formulas over vocabulary $\sigma$;
\item $\sem(\cF)$, for a theory $\cF\subseteq L$
 over class $\C$ of constraints and $\pair{\sigma_c}{\C}$-denotation,  
is the set of extended models of~$\cF$,
\end{enumerate}
captures SMT formulas.
We call this logic SMT-logic (over constraints $\C$).
If we restrict elements of $L$ to be conjunctions of literals 
\beq a_1\wedge\dotsc\wedge a_\ell\wedge\ \neg  a_{\ell+1}\wedge\dotsc\wedge\ \neg  a_m.
\eeq{eq:wcbodyf}
then we call this extended logic a  {\em Restricted SMT or RSMT-logic}. 

\begin{example}\label{ex:smt} Example~\ref{ex:casp} defines $\sigma_1$, $\C_1$, and $({\sigma_1}_c,\C_1)$-denotation considered here.
Let clauses 
$$\{a\vee b,~\neg a,~\neg a\vee \neg |x\neq 0|\}$$
over $\sigma_1$, $\C_1$, and $({\sigma_1}_c,\C_1)$-denotation 
form an SMT-logic theory~$H$. Recall how (i) Example~\ref{ex:casp} is a continuation of Example~\ref{ex:0}, where domain of variable $x$ was specified as $\{0,1,2\}$ and (ii) valuations $\nu_0,\nu_1,\nu_2$ are chosen to assign $x$ to $0$, $1$, $2$, respectively.
There are three extended interpretations that are extended models to the considered SMT-logic theory~$H$:
$$(\{b\},\nu_0),~~ (\{b,|x\neq 0|\},\nu_1),~~ (\{b,|x\neq 0|\},\nu_2).$$

\end{example}

\paragraph{Integer  CSP as Extended Logic}
An {\em integer  expression} 
has the form
$$
b_1 x_1 +\cdots + b_n x_n, 
$$
where $b_1,\dots,b_n$ are integers  $\mathbb{Z}$ and $x_1,\dots,x_n$ are  variables over $\mathbb{Z}$. 
When $b_i=1$ ($1\leq i\leq n$) we may omit it from the expression.
We call a constraint  {\em integer} when it is over specification whose domain is $\mathbb{Z}$ and is encoded implicitly via 
 the form
$
e\bowtie k,$
where $e$ is an integer expression, $k$ is an integer, and $\bowtie$ belongs to $\{<,>,\leq,\geq,=,\neq\}$.  
We call a CSP an {\em integer  constraint satisfaction problem}  when it is composed of integer  constraints. 

Integer  CSPs can be captured by extended logic theories. 
Indeed, an extended logic  $(L,\emptyset,\mathbb{Z},V,\sem)$,
where
\begin{enumerate}
\item $L$ is the set of integer  constraints over $\spn{V}{\mathbb{Z}}$;
\item $\sem(\cF)$, for a theory $\cF\subseteq L$,
is the set of pairs $(\emptyset,\nu)$, where evaluation $\nu$ is a solution to~$\cF$,
\end{enumerate}
captures integer  CSPs.
We call this logic an {\em I-CSP-logic}. If the domain of this extended logic is that of nonnegative integers $\mathbb{Z}^+$ then we call this logic a {\em nonnegative I-CSP-logic}.

\section{Extended Weighted Abstract Modular Systems or EW-Systems}\label{sec:ewsystems}

\cite{lt2015} propose (model-based) abstract modular systems or AMS that allow us to construct heterogeneous systems based of ``modules'' stemming from a variety of logics. 
We now  generalize their framework by incorporating the notion of an extended logic.

\begin{definition}[Extended Modules (or e-modules) and their Models]
Let  $\cL+$
be an extended logic.
 A theory of $\cL+$, that is, a subset of the 
language $L_{\cL+}$ is called an \emph{extended (model-based) $\cL+$-module} 
(or an \emph{e-module}, if the explicit 
reference to its logic is not necessary). 

Let $T$ be an extended $\cL+$-module.
An extended interpretation  $(I,\nu)$ over $\sigma_{\cL+}$, $\D_{\cL+}$, $\V_{\cL+}$
 is an \emph{extended model} of  $T$, whereas $I$ is a {\em model} of $T$
if $(I,\nu)\in\sem_{\cL+}(T)$. 

By 
$L_T$, 
$\sigma_T$, 
$\D_T$,
$\V_T$, and $\sem_T$
we refer to the elements $L_{\cL+},\sigma_{\cL+},\D_{\cL+}, \V_{\cL+},$ and $\sem_{\cL+}$ of module $T$ logic $\cL+$, respectively.
\end{definition}
As before, we use words theory and modules interchangeably. 
Furthermore, for a theory/module in  SMT-logic we often refer to these as  SMT formulas.   For a theory/module in  CAS-logic we  refer to it as a  CAS program.

For an interpretation $I$, by $I_{|\sigma}$ we denote an interpretation over  vocabulary $\sigma$ constructed  from~$I$ by  dropping  all  its members not in $\sigma$. 
For a set $V$ of variables and an evaluation $\nu$ defined over some superset of $V$, by $\nu_{|V}$ we denote an
evaluation over $V$  constructed from $\nu$ so that $\nu_{|V}(v)=\nu(v)$ for any variable $v$ in $V$.
For example, let $V$ be the set of variables $\{x,y\}$, and evaluation $\nu$ defined over $\{x,y,z\}$ assigns~$x$ and~$y$ to~$1$ and variable $z$ to $2$.
Evaluation $\nu_{|V}$ is defined on domain $V$ and assigns both of its variables value~$1$.

We now generalize the notion of an extended model to vocabularies and evaluations that go beyond the one of a considered module in a straight forward manner.  For an e-module $T$ and an extended interpretation $(I,\nu)$
over $\sigma$, $V$, $D$ so that $\sigma_T\subseteq \sigma$, $\V_T\subseteq V$,  $\D_T\subseteq D$,  we say that  
$(I,\nu)$ is an {\em extended model} of $T$, denoted $(I,\nu)\models T$, if $(I_{|{\sigma_T}},\nu_{|\V_T})\in \sem_T$. We can generalize the concept of a model in a similar way. 

We call extended logics
$(L,\sigma,D, V,\sem)$ and
$(L',\sigma',D', V',\sem')$
(and, respectively e-modules in these logics) {\em coherent} if
$D=D'$, whenever $V\cap V'\neq \emptyset$. In other words if  e-logics are coherent and they share variables then the domains of these e-logics  coincide.

\begin{definition}[Extended Abstract Modular Systems (EAMSs) and their models]
A set of coherent e-modules, possibly in different logics, over different
vocabularies and/or variables, is an \emph{extended (model-based) abstract modular system (EAMS)}. 
For an extended abstract modular system $\cH$, 
\begin{itemize}
\item the union of the 
vocabularies of the logics of the modules in $\cH$ forms the \emph{vocabulary} 
of~$\cH$, 
denoted by $\sigma_\cH$, 
\item  the union of the 
variables of the logics of the modules in $\cH$ forms the  \emph{set of variables} 
of~$\cH$, 
denoted by $\V_\cH$,  
\item the union of the 
domains of the logics of the modules in $\cH$ forms the  \emph{domain} 
of~$\cH$, 
denoted by $\D_\cH$.
\end{itemize}
An extended interpretation $(I,\nu)$
over $\sigma_\cH$, $\V_\cH$, $\D_\cH$ 
is an \emph{extended model} of $\cH$ whereas $I$ is a {\em model} of $\cH$, when for every module $B\in\cH$,  $(I,\nu)$ is an extended model of $B$. 
\end{definition}
When an EAMS consists of a single module $\{F\}$ we  identify it with module $F$ itself.
Just as  the concept of an extended logic presented here is a generalization of logic by Brewka and Eiter, the concept of EAMS  is a generalization of AMS by Lierler and Truszczynski~\shortcite{lt2015}.

\paragraph{Extended W-systems}
In practice, we are frequently interested not only in identifying models of a given logical formulation of a problem (hard fragment), but also identifying models that are deemed optimal according to some criteria (soft fragment). Frequently, multi-level optimizations are of interest. 
Lierler and Truszczynski~\shortcite{lt2015} argued how AMS and, consequently, EAMS as its generalization are geared towards capturing heterogeneous solutions for formulating hard constraints. 
Lierler~\shortcite{lie21} used abstract modular systems to formulate a concept of  w-systems that enable   soft constraints.  W-systems are adequate to capture the MaxSAT/MinSAT family of problems~\citeb{rob10} as well as logic programs with  weak constraints~\citeb{alv18}. W-systems  provide us with means of studying these distinct logic frameworks under a unified viewpoint.  

\begin{definition}[Ew-conditions and their models]
An \emph{ ew-condition} in an extended logic $\cL+$  is
 a pair $(T,w;\vec{c}@l)$ --- 
  consisting of an $\cL+$ e-module $T$    and an expression $w;\vec{c}@l$, where  
  \begin{itemize}
  \item $w$ is an integer, 
  \item \vec{c} is a function from variables in $\V_T$ to reals, and
  \item $l$ is a positive integer.  
  \end{itemize}
We refer to integers $l$ and $w$ as {\em levels} and {\em weights}, respectively.  We refer to function $\vec{c}$ as {\em coefficients} function. 

Let $B$ be an ew-condition $(T,w;\vec{c}@l)$.
Intuitively, by 
$\sigma_B$, 
$\D_B$, and
$\V_B$
we refer to
$\sigma_T$, 
$\D_T$, 
and~$\V_T$,  respectively.
Let $(I,\nu)\in \Int(\sigma_T,\V_T,\D_T)$ be 
an extended interpretation; it  
is an \emph{extended model} of $B$, denoted $(I,\nu)\models B$, when 
  $(I,\nu)$ is an extended model of $T$, also in this case $I$ is called a {\em model}, denoted $I\models B$.

The notion of an (extended) model for ew-condition is generalized to vocabularies and evaluations that go beyond the one of a considered ew-condition in a straight forward manner (as it was done earlier for the case of e-modules). For an extended interpretation $(I,\nu)$
over $\sigma$,~$V$,~$D$ so that $\sigma_B\subseteq \sigma$,~$\V_B\subseteq V$,~$\D_B\subseteq D$,  we say that  
$(I,\nu)$ is an {\em extended model} of ew-condition $B$, denoted $(I,\nu)\models B$, if $(I_{|{\sigma_T}},\nu_{|\V_T})\models B$. The concept of a model is generalized in a similar way. 
\end{definition}

Intuitively, the role of  weights $w$ in ew-condition $B=(T,w;\vec{c}@l)$ is to distinguish the quality of models/extended models of $B$ given their propositional part;   the role of the coefficient $\vec{c}$ is to distinguish the quality of extended models of $B$ given their evaluation  part. These intuitions become more apparent  in the next definition,  where we associate ``cost'' expressions with the models of ew-conditions. We elaborate more on these expressions after their definitions.

\begin{definition}[Cost expressions for (extended) interpretations of ew-conditions]
Ler $B$ be an ew-condition.

For an extended interpretation~$(I,\nu)$
over $\sigma$,~$V$,~$D$ so that $\sigma_B\subseteq \sigma$,~$\V_B\subseteq V$,~$\D_B\subseteq D$,
a mapping $\br{(I,\nu)\models B}$ is defined as 
\beq 
\br{(I,\nu)\models B}=\begin{cases}
  \displaystyle{\sum_{x\in \V_B}{\nu(x)\cdot \vec{c}(x)}}  &\hbox{when $(I,\nu)\models B$}  \\
  0&\hbox{otherwise.}  
\end{cases}
\eeq{eq:isatext}

For an interpretation~$I$
over $\sigma$ so that $\sigma_B\subseteq \sigma$, 
a mapping $\br{I\models B}$ is defined as follows
\beq \br{I\models B}=\begin{cases}
  w&\hbox{when $I\models B$, }  \\
  0&\hbox{otherwise.}  
\end{cases}
\eeq{eq:isat}

\end{definition}
\vspace{2mm}

We view expressions  $\br{(I,\nu)\models B}$ and $\br{I\models B}$ as costs  associated with two distinct parts of an (extended) interpretation of  a considered ew-condition~$B$. Indeed, the former expression accounts for  {\em ``the quality'' of  an evaluation} associated with an extended interpretation and utilizes the coefficient of ew-condition $B$ to compute that. The later expression accounts for {\em``the quality'' of ``logical/propositional part''} of an  interpretation by considering the weight of  ew-condition $B$.  
It is easy to see that non-zero values for costs are associated only with (extended) interpretations that are also models.
These cost expressions are used in formulating the definitions of \underline{optimal} (extended) models of  ew-systems defined next. These ew-systems take ew-conditions as the tool for distinguishing quality of their (extended) models. 
Formulas~\eqref{eq:isatext} and \eqref{eq:isat} are used within expressions 
\eqref{eq:condeqlmin} and \eqref{eq:condeqlmin2} utilized in the definitions of optimal models and optimal extended models of ew-systems, respectively.
In the sequel, 
Examples~\ref{ex:wsmt},~\ref{ex:ewsmt},~\ref{ex:omt} illustrate the use of instances of cost expressions~\eqref{eq:isatext} and \eqref{eq:isat} within formulas~\eqref{eq:condeqlmin} and \eqref{eq:condeqlmin2} .

Before introducing the key concept of this paper --  ew-systems -- we present a number of useful abbreviations.
We identify ew-conditions of the form 
\begin{itemize}
\item $(T,w;\vec{c}@1)$ with expressions $(T,w;\vec{c})$: i.e., when the level is missing it is considered to be~$1$. 
\item $(T,w;\emptyset@l)$ with expressions $(T,w@l)$: i.e., when the coefficients function is empty. 
\item $(T,w;\vec{c}@l)$ with expressions $(T,w@l)$, when the coefficients function $\vec{c}$ assigns $0$ to every element in its domain. 
\end{itemize}
For example, $(T,w)$ stands for ew-condition, whose level is~$1$ and whose coefficients function is either empty or assigns~$0$ to every element in its domain. 

For a collection $\cS$ of ew-conditions,  its 
\emph{vocabulary}, 
denoted by $\sigma_\cS$, its  \emph{set of variables}, 
denoted by $\V_\cS$,  its  \emph{domain}, 
denoted by $\D_\cS$ is defined following the lines of these concepts for EAMS. 
We say that ew-condition $(T,w;\vec{c}@l)$ is {\em coherent} with an EAMS $\cH$ if 
\begin{itemize}
    \item $\sigma_T\subseteq\sigma_\cH$,
    \item $\V_T\subseteq \V_\cH$, 
    \item given any module $H$ in $\cH$, 
    $\D_T=\D_H$ whenever $\V_T\cap \V_H\neq \emptyset$.
\end{itemize}

\begin{definition}[Ew-systems and their models]\label{def:omewsys}\nopagebreak 
A pair $(\cH,\cS)$ consisting of an EAMS $\cH$ and a set $\cS$ of ew-conditions so that every element in~$\cS$ is coherent with $\cH$
is called an {\em ew-system} ($\cH$ and $\cS$ intuitively stand for {\em hard} and {\em soft}, respectively).

Let $\cW=(\cH,\cS)$ be
an ew-system. The 
vocabulary of~$\cH$ forms the \emph{vocabulary} 
of $\cW$, denoted by $\sigma_\cW$. Similarly, $\V_\cW=\V_\cH$ and $\D_\cW=\D_\cH$.
An extended interpretation $(I,\nu)$ in $\Int(\sigma_\cW, \V_\cW, \D_\cW)$
is an \emph{extended model} of $\cW$, whereas $I$ is a {\em model}, if $(I,\nu)$ is an extended model of~$\cH$.
\end{definition}

For a level~$l$, by $\cW_l$ we denote the subset of $\cS$ that includes all ew-conditions whose level is~$l$.
By~$\level{\cW}$, we denote the set $
\{l\mid
 (T,w;\vec{c}@l)\in \cS\}$ of all levels associated with ew-system $\cW$.
For a level $l\in \level{\cW}$ by $\prec{l}$ we denote the least level in $\level{\cW}$ that is greater than~$l$ (it is obvious that for the greatest level in $\level{\cW}$,  $\prec{l}$ is undefined). For example, for levels in~$\{2, 6, 8, 9\}$,
$\prec{2}=6$, $\prec{6}=8$, and~$\prec{8}=9$.

\begin{samepage}
\begin{definition}[Optimal models of ew-systems]\label{def:oemewsys}\nopagebreak 
For  level $l\in\level{\cW}$, a model $I^*$ of ew-system $\cW$ is {\em $l$-optimal} if 
$I^*$ satisfies equation
\beq
I^*=\displaystyle{ arg\max_{I} {\sum_{B\in\cW_l}{ \br{I\models B}}}
},
\eeq{eq:condeqlmin}
where
\begin{itemize}
    \item $I$ ranges over models of $\cW$ if $l$ is the greatest level in $\level{\cW}$, 
\item $I$ ranges over $\prec{l}$-optimal models of $\cW$, otherwise.
\end{itemize}
We call a model  {\em $l$-min-optimal} if max is replaced by min in~\eqref{eq:condeqlmin} (and occurrences of word {\em optimal} are replaced by {\em min-optimal} in the definition; we drop this remark from the later similar definitions). 

A model $I^*$ of $\cW$ is {\em optimal} if $I^*$ is $l$-optimal model for every level $l\in\level{\cW}$.
A model $I^*$ of~$\cW$ is {\em min-optimal} if $I^*$ is $l$-min-optimal model for every level $l\in\level{\cW}$.
\end{definition}
\end{samepage}

We now provide intuitions for sub-expression 
\beq {\sum_{B\in\cW_l}{ \br{I\models B}}}
\eeq{eq:cost1}
of the formula~\eqref{eq:condeqlmin}. Given an interpretation $I$, formula~\eqref{eq:cost1} can be seen as a cost of this interpretation with respect to level $l$ of ew-system $\cW$.
This cost is computed by summing all the weights of the ew-conditions of $\cW_l$ for which this interpretation $I$ is a model. 
When ew-system $\cW$ contains only ew-conditions of a single level then  formula~\eqref{eq:cost1} equips us with the cost of the considered interpretation with respect to the overall system. 

\begin{definition}[Extended Optimal Models of Ew-systems]\label{def:oemews}
For  level $l\in\level{\cW}$, an \underline{extended model} $(I^*,\nu^*)$ of ew-system~$\cW$ is {\em l-optimal} if 
$(I^*,\nu^*)$ satisfies equation
\beq
(I^*,\nu^*)=
\displaystyle{ arg\max_{(I^*,\nu^*)} {
\sum_{B\in\cW_l}{ (\br{I\models B} + \br{(I,\nu)\models B})}}
},
\eeq{eq:condeqlmin2}
where
\begin{itemize}
    \item $(I,\nu)$ ranges over extended models of $\cW$ if $l$ is the greatest level in $\level{\cW}$, 
\item $(I,\nu)$ ranges over $\prec{l}$-optimal extended models of $\cW$, otherwise.
\end{itemize}
We call a model  {\em $l$-min-optimal} if max is replaced by min in the equation above.

An extended model $(I^*,\nu^*)$ of $\cW$ is {\em optimal} if $(I^*,\nu^*)$ is $l$-optimal model for every level $l$ in~$\level{\cW}$.
An extended model $(I^*,\nu^*)$ of $\cW$ is {\em min-optimal} if $(I^*,\nu^*)$ is $l$-min-optimal model for every level~$l$ in~$\level{\cW}$.
\end{definition}

We now provide intuitions for sub-expression 
\beq
\sum_{B\in\cW_l}{ (\br{I\models B} + \br{(I,\nu)\models B})}
\eeq{eq:cost2}
of the formula~\eqref{eq:condeqlmin2}. Given an extended interpretation $(I,\nu)$, formula~\eqref{eq:cost2} can be seen as a cost of this interpretation with respect to level $l$ of ew-system $\cW$.
This cost is computed by 
considering  all the 
 ew-conditions of $\cW_l$ for which this extended interpretation 
is a model and 
summing all their weights 
together with the values provided by linear expression formed within cost expression provided at the first line of~\eqref{eq:isatext}. 
When ew-system $\cW$ contains only ew-conditions of a single level then  formula~\eqref{eq:cost2} equips us with the cost of the considered extended interpretation with respect to the overall system.

If we compare the notion of a optimal/min-optimal model versus a optimal/min-optimal extended model then the former does not take into account the numeric values associated with evaluations corresponding to extended models associated with the considered model; whereas the latter combines the quality of both parts of extended model.

\cite{lie22} noted that the definition of optimal models in terms of ``$arg\max$''-equation comes from the traditions of literature related to MaxSAT problem.
In answer set (logic) programming community, the conditions on optimality of answer sets is stated in terms of ``domination'' relation. 
Here we follow the steps by~\cite{lie22} and provide an alternative definition to optimal models of ew-systems in terms of domination.

\begin{definition}[Optimal models of ew-systems]\label{def:optimalmodelewsysASP}
Let $I$ and $I'$ be models of ew-system $\cW$.
Model $I'$ {\em min-dominates}  $I$ if 
there exists a level $l\in\level{\cW}$ such that
following conditions are satisfied:
\begin{enumerate}
\item for any level $l'>l$  the following equality holds
$$
\displaystyle{ \sum_{B\in\cW_{l'}}{ \br{I\models B}}}
=
\displaystyle{ \sum_{B\in\cW_{l'}}{ \br{I'\models B}}}
$$
\item the following inequality holds for level $l$
$$
\displaystyle{ \sum_{B\in\cW_l}{ \br{I'\models B}}}
<
\displaystyle{ \sum_{B\in\cW_l}{ \br{I\models B}}}
$$ 
\end{enumerate}
Model $I'$ {\em max-dominates}  $I$ if 
we change less-than symbol by greater-than symbol in the inequality of Condition~\ref{l:cond2}.

A model $I^*$ of $\cW$ is {\em optimal}  if there is no model~$I'$ of $\cW$ that max-dominates $I^*$. 
A model $I^*$ of $\cW$ is {\em min-optimal}  if there is no model~$I'$ of $\cW$ that min-dominates $I^*$. 
\end{definition}

\begin{definition}[Optimal extended models of ew-systems]\label{def:optimalemodelewsysASP}
Let $(I,\nu)$ and $(I',\nu')$ be extended models of ew-system $\cW$.
Extended model $(I',\nu')$ {\em min-dominates}  $(I,\nu)$ if 
there exists a level $l\in\level{\cW}$ such that
following conditions are satisfied:
\begin{enumerate}
\item for any level $l'>l$  the following equality holds
$$
\displaystyle{ 
{\sum_{B\in\cW_{l'}}{ (\br{I\models B} + \br{(I,\nu)\models B})}}
}
=
\displaystyle{ 
{\sum_{B\in\cW_{l'}}{ (\br{I'\models B} + \br{(I',\nu')\models B})}}
}
$$
\item the following inequality holds for level $l$
$$
\displaystyle{ 
{\sum_{B\in\cW_{l}}{ (\br{I'\models B} + \br{(I',\nu')\models B})}}
}
<
\displaystyle{ 
{\sum_{B\in\cW_{l}}{ (\br{I\models B} + \br{(I,\nu)\models B})}}
}
$$ 
\end{enumerate}
Extended model $(I',\nu')$ {\em max-dominates}  $(I,\nu)$ if 
we change less-than symbol by greater-than symbol in the inequality of Condition~\ref{l:cond2}.

An extended model $(I^*,\nu^*)$ of $\cW$ is {\em optimal}  if there is no extended model~$(I',\nu')$ of $\cW$ that max-dominates $(I^*,\nu^*)$. 
An extended model $(I^*,\nu^*)$ of $\cW$ is {\em min-optimal}  if there is no extended model~$(I',\nu')$ of $\cW$ that min-dominates $(I^*,\nu^*)$. 
\end{definition}

\begin{proposition}\label{prop:eqdefs}
Definitions~\ref{def:oemewsys} and~\ref{def:optimalmodelewsysASP} are equivalent.
Definitions~\ref{def:oemews} and~\ref{def:optimalemodelewsysASP} are equivalent.
\end{proposition}

\section{Instances of Ew-Systems}

\subsection{MaxSMT-family and Optimization Modulo Theories}\label{sec:5.1} 
\cite{lie21,lie22} illustrated the utility of w-systems --- a precursor of ew-systems ---
by using these to capture the definitions of {\em MaxSAT}, {\em weighted MaxSAT}, and {\em partially weighted MaxSAT} (or, {\em pw-MaxSAT})~\citeb{rob10}. Here we look into capturing  {\em weighted MaxSMT} and {\em pw-MaxSMT}.

We start by introducing some useful notation.
For 
a vocabulary $\sigma$, a specification  $\spn{V}{D}$,
and an e-logic $\cL+$ so that \mbox{$\sigma_{\cL+}=\sigma$},
\mbox{$\D_{\cL+}=D$}, \mbox{$\V_{\cL+}=V$}, 
an e-module $T_{\cL+}$ is called
{\em $\sigma,V,D$-theory/$\sigma,V,D$-module}  when 
\mbox{$sem(T_{\cL+})=Int(\sigma,V,D)$}.
In other words, any extended interpretation in $Int(\sigma,V,D)$ is an extended model of a~$\sigma,V,D$-theory, 
and, consequently, any interpretation in $Int(\sigma)$ is a model of a~$\sigma,V,D$-theory.
Note how specific language of e-logic of  
a $\sigma,V,D$-theory becomes immaterial.
Thus, we allow ourselves to denote an arbitrary  $\sigma,V,D$-theory by
 $T_{\sigma,V,D}$ disregarding the reference to its e-logic. Also recall that the (extended) interpretations/models of any theory are generalized to signatures that go beyond original signature of the considered theory.


\begin{definition}[Weighted MaxSMT problems and their solutions]\label{def:solwmsmt}
A {\em weighted MaxSMT problem}~\citeb{nie06a} over vocabulary $\sigma$, class $\C$ of constraints and $\pair{\sigma_c}{\C}$-denotation 
is defined as a set $S$ of pairs $\pair{\cF}{w}$, where~$\cF$ is an SMT formula over $\sigma$, $\C$, and 
$\pair{\sigma_c}{\C}$-denotation so that $\cF$ is a clause, and $w$ is a positive integer\footnote{\cite{nie06a} allow 
$w$  to be a positive real number.}. 
An interpretation $I^*\in Int(\sigma)$  is a  {\em solution} to weighted MaxSMT problem $S$, when it satisfies the equation 
$$\displaystyle{I^*=\argmax_{I}{\sum_{(\cF,w)\in S}{w\cdot \br{I\models \cF}}}},\hbox{ where}\\
$$
\beq
\br{I\models \cF}=\begin{cases}
  \displaystyle{1}  &\hbox{when $I\models \cF$}  \\
  0&\hbox{otherwise}  
\end{cases}
\eeq{eq:smteq}
and $I$ ranges over all interpretations in $Int(\sigma)$.
\end{definition}

\begin{example}\label{ex:wsmt}
Consider specification $s_1$, class $\C_1$ of constraints consisting of  $c_1$ and its complement,
vocabulary~$\sigma_1$, and 
$({\sigma_1}_c,\C_1)$-denotation from Example~\ref{ex:casp}.
Set
$$\{(a\vee b,2), (\neg a,3), (\neg a\vee \neg|x\neq 0|,1)\}$$
exemplifies a weighted MaxSMT problem. Its solutions are
$\{b\}$ and $\{b,|x\neq 0|\}$.
\end{example}

\begin{proposition}\label{prop:wmaxsmt}
Let $S$ be a weighted MaxSMT problem.
The optimal models of ew-system $(T_{\sigma_S,\V_S,\D_S},S)$ ---
where each element in $S$ is understood as an  ew-condition, whose theory is in SMT-logic ---
 form the set of solutions for weighted MaxSMT problem~$S$. 
\end{proposition}
Proposition~\ref{prop:wmaxsmt} allows us to view ew-systems as a mean to  state  definitions of semantics of various formalisms captured by e-logics  with optimizations in a uniform way. Indeed, we can formulate a definition of semantics of weighted MaxSMT problem by means of its relation to a respective ew-system: {\em 
For a {weighted MaxSMT problem} $S$
over vocabulary $\sigma$, class $\C$ of constraints and $\pair{\sigma_c}{\C}$-denotation,  
an optimal model of ew-system $(T_{\sigma_S,\V_S,\D_S},S)$ is a {\em solution} to~$S$.}
In the sequel, we often take this approach to present the definitions of various optimization formalisms and their possible extensions. This allows us to bypass the complexity of varying terminology and notation coming from papers that introduce these formalisms. It is also worth to reiterate that we speak of papers stemming from various research communities that have their own traditions.

Note how in ew-system $(T_{\sigma_S,\V_S,\D_S},S)$, where $S$ is a weighted MaxSMT problem, 
\begin{itemize}
    \item levels of all ew-conditions are identical (set to 1) and 
    \item the coefficients function assigns all variables to $0$. 
\end{itemize}
These observations lend themselves to a natural generalization of weighted MaxSMT problem.
\begin{definition}[Generalized weighted MaxSMT problems and their (extended) solutions]
A {\em generalized weighted MaxSMT problem} over class $\C$ of constraints is defined as a set $S$  of  ew-conditions in SMT-logic over $\C$. Optimal models and optimal extended models of ew-system $(T_{\sigma_S,\V_S,\D_S},S)$ form {\em solutions} and {\em extended solutions} of $S$. 
\end{definition}
Note how extended solutions of $S$ take into account the quality of evaluation that is part of extended model. Also, levels become naturally incorporated into the framework. 


\begin{example}\label{ex:ewsmt}
Let us continue Example~\ref{ex:wsmt}. 
Let $f_x$ be a coefficients function defined as $f_x(0)=10$, $f_x(1)=100$, $f_x(2)=1000$.
The following set~~ 
$\{(a\vee b,2), (\neg a,3), (\neg a\vee \neg|x\neq 0|,1;f_x)\}
$~~
over $\sigma_1$, $\C_1$, and $({\sigma_1}_c,\C_1)$-denotation forms 
 a generalized weighted MaxSMT problem. Its solutions are
$\{b\}$ and $\{b,|x\neq 0|\}$. Its extended solution is $(\{b,|x\neq 0|\},\nu_2)$, where $\nu_2$ is introduced in Example~\ref{ex:casp}. If we redefine $f_x$ as follows $f_x(0)=1000$, $f_x(1)=100$, $f_x(2)=10$, then the extended solution to considered problem is  $(\{b\},\nu_0)$.

\end{example}

We now  define/generalize partially weighted MaxSMT (pw-MaxSMT) problem inspired by  the notion of  partially weighted MaxSAT problem~\citeb{rob10}.
\begin{definition}[Generalized partially weighted MaxSMT problems and their (extended) solutions]
A {\em generalized partially weighted MaxSMT (or gpw-MaxSMT) problem} over vocabulary $\sigma$, class $\C$ of constraints over specification $\spn{V}{D}$ and $\pair{\sigma_c}{\C}$-denotation  is defined as ew-system  $(H,S)$, where $H$ is a theory in SMT-logic over  $\C$
 and $S$ is a collection of ew-conditions whose theories are in SMT-logic over $\C$.
Formula~$H$ is referred to as {\em hard} problem fragment, whereas  $S$ forms {\em soft} problem fragment.
 Optimal models and optimal extended models of ew-system $(H,S)$ form {\em solutions} and {\em extended solutions} of~$(H,S)$.
\end{definition}

\subsubsection{\nuz approach}
As mentioned in the introduction, often the complete spectrum of capabilities of automated reasoning systems is described in papers and tutorials by means of examples appealing to intuitions of the readers. We argue that the presented framework provides convenient means to make such presentations formal.
For example, \cite{nuz}, the authors of SMT solver~\nuz, state
that 
\begin{quote}
\nuz extends the functionality of Z3~\citeb{mou08} to include optimization objectives. It allows
users to solve SMT constraints and at the same time formulate optimality criteria
for the solutions. \dots    
\end{quote}
 They then specify that \nuz adds to SMT-LIB~\citebb{BarFT-SMTLIB}{BarST-SMT-10} -- standard descriptions/language of background theories used in SMT systems -- 
a command of the form  

{\tt (assert-soft F [:weight n])}. 

\noindent
This command is said to assert soft constraint $F$, optionally with an integer weight $n$.\footnote{The authors also allow a keyword for decimal weights, yet at the time of writing this paper the system z3-4.8.12, the latest available incarnation of z3 with the functionality of \nuz gave out an error message "invalid keyword argument" suggesting that this feature is no longer supported.} If no
weight is given, the default weight is $1$. 
Code similar to the one presented in  LHS of Figure~\ref{fig:fig1} and the respective output of system \nuz presented in RHS of the figure are used by~\cite{nuz} to 
illustrate a use of soft constraints.
\begin{figure}[h]
\small
    \centering
    \begin{multicols}{2}
    \begin{verbatim}
LHS:
(declare-fun x () Int)
(declare-fun y () Int)
(define-fun a1 () Bool (> x 0))
(define-fun a2 () Bool (< x y))
(assert (=> a2 a1))
(assert-soft a2 :weight 3)
(assert-soft (not a1) :weight 5)
(check-sat)
(get-model)


RHS:
sat
 ((define-fun a1 () Bool
    (not (<= x 0)))
  (define-fun x () Int
    0)
  (define-fun a2 () Bool
    (not (<= y x)))
  (define-fun y () Int
    0))
    \end{verbatim}
    \end{multicols}
    \normalsize
    \caption{LHS: SMT-LIB code suggesting to maximize $3 \cdot a2 + 5 \cdot \overline{a1}$.  RHS: \nuz finds
a solution where $x = y= 0$. }
    \label{fig:fig1}
\end{figure}

We now elaborate on this example by~\cite{nuz} to uncover its implicit assumptions about readers intuitions and interpretations of the seen code and its behaviour. 
Given code in Figure~\ref{fig:fig1} (LHS) as an input, system z3-4.8.12 produces a model where $x$ and $y$ are both assigned to $0$ suggesting that $a1$ and $a2$ are assigned to $false$. 
The authors of \nuz informally state that code in Figure~\ref{fig:fig1}~(LHS) maximizes expression 
\beq
3 \cdot a2 + 5 \cdot \overline{a1}.
\eeq{ex:max}
To make this claim precise we recall that system \nuz looks for models of a provided SMT formula specified in its {\tt assert} statements --- an implication $a2\rightarrow a1$ in this running example, where $a1$ and $a2$ are constraint atoms associated with (via {\tt define-fun} statements) integer  constraints $x>0$ and $x<y$, respectively. Expressions 
$a2$ and $\overline{a1}$ in~\eqref{ex:max} should be interpreted in relation with some model $X$ for the encoded SMT formula so that, for instance, 
\begin{itemize}
    \item $a2$ is mapped to $0$ if  model $X$ assigns $a2$ to $false$, and $1$ otherwise; similarly 
\item $\overline{a1}$ is mapped to $1$ if  model $X$ assigns $a1$ to $false$, and $0$ otherwise. 
\end{itemize}
It is easy to see that there are three models to the implication  $a2\rightarrow a1$. Furthermore, for each of 
these three models there is an evaluation that maps integer variables $x$ and $y$ into values that satisfy integer  constraints associated with constraint atoms $a1$ and $a2$. Thus, all these models form solutions to the SMT formula specified in Figure~\ref{fig:fig1}.
These models together with the respective value of expression~\eqref{ex:max} follow: 
$$
\ba{lll}
a1 &a2&3 \cdot a2 + 5 \cdot \overline{a1} \\
false&false& 5\\
true&true& 3\\
true&false&0
\ea
$$
It is easy to see that the model listed first maximizes value of expression~\eqref{ex:max} in comparison to other models.

Generalized partially weighted MaxSMT problems and their solutions can be seen as  formal specification of the language supported  by the SMT solver \nuz. Indeed, the statements that follow the key word {\tt assert} correspond to hard problem fragment; whereas the statements that follow the key word 
{\tt assert-soft} form soft problem fragment.
In other words, 
we view \nuz code as a specification of a gpw-MaxSMT problem, whereas we consider
system \nuz to compute solutions for given gpw-MaxSMT problems. 
Let us go back to the \nuz code
 in Figure~\ref{fig:fig1}. It specifies the gpw-MaxSMT problem over vocabulary 
$\{a1,a2\}$,
class of integer  constraints over specification $\spn{\{x,y\}}{\mathbb{Z}}$, and denotation that maps
$a1$ and $a2$ into integer constraints $x>0$ and $x<y$, respectively; so that hard problem fragment consists of propositional formula $a2\rightarrow a1$, and   soft problem fragment consists of ew-conditions $(a2,3)$ and $(\neg a1,5)$.
It is easy to see that interpretation mapping $a1$ and $a2$ to $false$ forms a solution to the constructed gpw-MaxSMT problem.

Above, we complemented an earlier description of \nuz input-output provided by an example with  formal specification of these entities. It is also due to note that gpw-MaxSMT problems are more general than SMT-LIB specifications supported by \nuz. For example, the notion of a level is present in gpw-MaxSMT specifications of its soft fragment. Also, the notion of an extended solution is defined for gpw-MaxSMT problems. This may provide inspirations for possible extensions to system \nuz.

\subsubsection{Cost-variable approach}
\cite{seb12} observed that MaxSMT and its variants encapsulated by gpw-MaxSMT when (non-extended) solutions are considered support optimizations that focus on "propositional part" of a problem encoded within SMT framework. They then proposed an alternative approach that ranks extended models of an SMT formula by value of one of its variables occurring in constraints associated with an SMT formula disregarding any information of a model from a propositional side. In particular, the value of that variable is to be minimized.

We now use ew-systems to capture an  optimization satisfiability module problem (OMT) by~\cite{seb12}.
To proceed to the definition recall the notion of a \hbox{$\sigma,V,D$-theory}, denoted $T_{\sigma,V,D}$, introduced in the beginning of Section~\ref{sec:5.1}.
\begin{definition}[OMT problems and their solutions]
An {\em OMT problem}
over vocabulary $\sigma$, class $\C$ of constraints over specification $\spn{V}{D}$ and $\pair{\sigma_c}{\C}$-denotation and variable $v$ occurring in $V$ is defined as ew-system  $(H,S)$, where $H$ is a theory in SMT-logic over 
$\sigma$, $\C$, and $\pair{\sigma_c}{\C}$-denotation and $S$ consists of a single ew-condition of the form 
$(T_{\emptyset,\{v\},D},0;(v)\mapsto 1)$; We call min-optimal extended models of this ew-systems {\em solutions to OMT}.
\end{definition}
The use of ew-condition $(T_{\emptyset,\{v\},D},0;(v)\mapsto 1)$ in this definition as the only criterion for optimization captures the idea of OMT that relies on the choice of a single variable in constraints to be minimized. Indeed, the first component $T_{\emptyset,\{v\},D}$ of this ew-condition is such that any (extended) interpretation of theory $H$ in SMT-logic with variable $v$ occurring in it is a model of $T_{\emptyset,\{v\},D}$. The cost of any extended interpretation will be identified with the value of $x$ assigned by this interpretation based on 
formula~\eqref{eq:cost2} used in  the definition of the min-optimal extended model (for this instance, $\cW_l$ in \eqref{eq:cost2} is a singleton composed of $T_{\emptyset,\{v\},D}$). Indeed,
take $(I,\nu)$ be an arbitrary model of $B=T_{\emptyset,\{v\},D}$, then  $\br{I\models B}=0$ and  $\br{(I,\nu)\models B}=\nu(x)$; the sum of these values used within formula~\eqref{eq:cost2}  amounts to  $\nu(x)$. This value is then used to decide on min-optimality of the model within expression~\eqref{eq:condeqlmin2}, where $max$ is replaced by $min$.

\begin{samepage}
\begin{example}\label{ex:omt}\nopagebreak 
Example~\ref{ex:smt} defines 
an SMT-logic theory~$H$ that has 
three extended models 
$$(\{b\},\nu_0),~~ (\{b,|x\neq 0|\},\nu_1),~~ (\{b,|x\neq 0|\},\nu_2),$$
where
valuations $\nu_0,\nu_1,\nu_2$ are chosen to assign $x$ to $0$, $1$, $2$, respectively.
A sample OMT problem follows
$$
\Big(H,\big\{\big(T_{\emptyset,\{x\},\{0,1,2\}},0;(x)\mapsto 1\big)\big\}\Big)
$$
A solution for this OMT problem is an extended interpretation $(\{b\},\nu_0)$ -- the one that minimizes the value of $x$. 
\end{example}
\end{samepage}

\subsection{Integer Linear Programming}
An {\em integer linear program (IL-program)}~\citeb{pap82} was given in the introduction in~\eqref{eq:ilp}.
It is easy to see that the statement after the word {\tt maximize} is an integer expression, whereas statements after the words {\tt subject to} are integer  constraints. We now provide an alternative definition to integer linear programs.

\begin{definition}
An IL-program is an ew-system of the form
$(H,S)$, where $H$ is  a theory in a nonnegative I-CSP-logic  and $S$ is an ew-condition  $(T_{\emptyset,\V_H,\mathbb{Z^+}},0;\vec{c})$, where $\vec{c}$ is a coefficient function mapping variables in $\V_H$ to integers $\mathbb{Z}$. Extended optimal models to this system form {\em solutions to an IL-program}. 
\end{definition}
Intuitively,  theory $H$ corresponds to statements following {\tt subject to} in~\eqref{eq:ilp}, while 
the ew-condition $S$ captures a statement following the word {\tt maximize} in~\eqref{eq:ilp}.

\subsection{Optimizations in CAS Programs}
We now turn our attention to constraint answer set programming. At first, we review optimization statements of answer set programming.
We illustrate how they can be understood within CASP framework.
For this task we utilize ew-systems: just as we utilized ew-systems to capture generalization from pw-MaxSAT to pw-MaxSMT realm.
We then look into {\sc clingcon} style optimization statements native to CASP framework. We illustrate that ew-systems are general enough to encapsulate {\sc clingcon} programs with optimizations.
Systems {\sc clingcon-2}~\citeb{ost12} and {\sc clingcon-3}~\citeb{ban17}  vary in the syntax of the languages they accept and algorithmic/implementation details. In case of this paper it is interesting to look at these systems in separation as they provide different support for optimization statements.

\subsubsection{Optimizations in Logic Programming}\label{sec:olp}  
  We now review 
 a definition of a logic program with weak constraints following the lines of~\cite{cal15}.
        A {\em weak constraint} has the form
\beq 
\wr a_1,\dotsc, a_\ell,\ not\  a_{\ell+1},\dotsc,\ not\  a_m[w@l],
\eeq{eq:wc}
where $m>0$ and $a_1,\ldots,a_m$ are atoms,  $w$ (weight) is an integer, and  $l$ (level) is a positive integer. 
In the sequel, we abbreviate expression
\beq
\wr  a_1,\dotsc, a_\ell,\ not\  a_{\ell+1},\dotsc,\ not\  a_m
\eeq{eq:wcbody}
occurring in~\eqref{eq:wc} as $D$ and identify it with the propositional formula
\beq a_1\wedge\dotsc\wedge a_\ell\wedge\ \neg  a_{\ell+1}\wedge\dotsc\wedge\ \neg  a_m.
\eeq{eq:wcbodyf}
An {\em optimization program} (or {\em o-program})  over vocabulary $\sigma$ is a pair~$(\Pi,W)$, where $\Pi$ is a logic program over  $\sigma$ and~$W$ is a finite set of weak constraints over $\sigma$.

Let $\cP=(\Pi,W)$ be an optimization program over vocabulary $\sigma$ (intuitively, $\Pi$ and~$W$ forms  {\em hard} and {\em soft} fragments, respectively). 
By $\level{\cP}$ we denote the set of all levels associated with optimization program $\cP$ constructed as 
$
\{l\mid\, D[w@l]\in W\}$.
Set~$X$ of atoms over $\sigma$ is an {\em answer set} of $\cP$ when it is an answer set of $\Pi$.
\begin{definition}[Optimal answer sets]\label{def:oascc3}
Let $X$ and $X'$ be answer sets of $\cP$.
Answer set $X'$ {\em dominates}  $X$ if 
there exists a level $l\in\level{\cP}$ such that
following conditions are satisfied:
\begin{enumerate}
\item\label{l:cond1} for any level $l'$ that is greater than $l$  the following equality holds
$$
\displaystyle{ \sum_{D[w@l']\in W}{w \cdot\br{X\models D}}
} =\displaystyle{\sum_{D[w@l']\in W}{w \cdot\br{X'\models D}}
}
$$
\item\label{l:cond2} the following inequality holds for level $l$
$$
\displaystyle{\sum_{D[w@l]\in W}{w \cdot\br{X'\models D}}
}<\displaystyle{ \sum_{D[w@l]\in W}{w \cdot\br{X\models D}}
} 
$$ 
\end{enumerate}
An answer set $X^*$ of $\cP$ is {\em optimal}  if there is no answer set~$X'$ of $\cP$ that dominates $X^*$. 
\end{definition}



We now exemplify the definition of an optimization program. Let $\Pi_1$ be logic program~\eqref{ex:slp}.
An optimal answer set of optimization program
\beq
(\Pi_1,\{\wr a,not\ b. -2@1\})
\eeq{eq:sampleop}
is $\{a\}$. 

It is worth noting that an alternative syntax is frequently used by answer set programming practitioners when they expresses optimization criteria:
\beq
\#minimize\{w_1@l_1:lit_1,\dots,w_n@l_n:lit_n \},
\eeq{minimize_statement}
where $lit_i$ is either an atom $a_i$ or an expression $not\ a_i$.
This statement stands for $n$ weak constraints
$$
\wr lit_1[w_1@l_1]~~\dots~~\wr lit_n[w_n@l_n].
$$
Similarly, statement
\beq
\#maximize\{w_1@l_1:lit_1,\dots,w_n@l_n:lit_n \},
\eeq{maximize_statement}
stands for $n$ weak constraints
$$
\wr lit_1[-w_1@l_1]~~\dots~~\wr lit_n[-w_n@l_n].
$$

\cite{lie21} illustrated how o-programs can be identified with w-systems -- a pre-cursor of ew-systems.

\subsubsection{{\sc clingcon-2} style optimizations}
Recall {\em minimize} statement~\eqref{minimize_statement}.
System {\sc clingcon-2} allows a user to write such statements using the following restrictions. All of these statements occurring in a program must either 
\begin{itemize}
        \item come with expressions $lit_i$ ($1\leq i\leq n$) constructed from regular atoms of CAS program, or 
\item come with expressions $lit_i$ ($1\leq i\leq n$) constructed from constraint variables stemming from irregular atoms of CAS program. Also,  $w_i@l_i$ expressions are dropped in this case.\footnote{In case, when this restriction is not satisfied the system {\sc clingcon-2} outputs the following message {\em ERROR: Can not optimize asp and csp at the same time!}}
\end{itemize}
Thus, a user might either impose optimization criteria that pertain regular atoms or constraint variables associated with irregular atoms but not both. At the same time we note that the {\sc clingcon-2} authors provide no declarative semantics for programs with optimizations, but rather explain the behavior of the systems by means of examples. In what follows we capture the semantics of two variants of {\sc clingcon-2} formally. We refer to programs  of {\sc clingcon-2}  supporting optimization statements over regular atoms  {\sc clingcon-2.1} programs. We refer to programs  of {\sc clingcon-2}  supporting optimization statements over constraint variables {\sc clingcon-2.2} programs. 

\paragraph{{\sc clingcon-2.1} programs.}
We now extend CAS programs with weak constraints in a similar way as MaxSMT extends SMT with "soft SMT clauses". 
\begin{definition}[Optimization CAS program and their optimal (extended) answer sets]
An {\em optimization CAS program} (or {\em oCAS-program})  over vocabulary $\sigma$, class $\C$ of constraints, and $(\sigma_c,\C)$-denotation is an ew-system~$(P,W)$, where~$P$ is a CAS program over  $\sigma$, $\C$ and $(\sigma_c,\C)$-denotation,  and~$W$ is a finite set of weak constraints over $\sigma$, where 
we  identify weak constraints of the form $D[w@l]$  with ew-condition $(D,w@l)$ in RSMT-logic
over 
$\sigma$,~$\C$, and $\pair{\sigma_c}{\C}$-denotation. We call oCAS-program {\em {\sc clingcon-2} style}, when there is an additional restriction on its weak constraints $W$ to be over signature $\sigma_r$.
Min-optimal models and min-optimal extended models of ew-system $(P,W)$ form {\em optimal answer sets and optimal extended answer sets} of  oCAS-program~$(P,W)$.
\end{definition}
Note how the syntax of weak constraints supports optimizations that consider ``propositional part'' of a problem encoded within CASP.
Indeed,  coefficients functions of ew-conditions in oCAS-programs are assumed to be the zero functions. Thus, it is easy to see that any optimal extended answer set $(X,\nu)$ of some CAS program is such, whenever $X$ is an optimal answer set of this program; the quality of evaluation $\nu$ of the extended answer set is immaterial. 
 Recall how ``minimize'' statements~\eqref{minimize_statement} are abbreviations for the set of weak constraints (see Section \ref{sec:olp}). In this regard,
{\sc clingcon-2.1} programs  are  captured by {\sc clingcon-2} style oCAS-programs, in other words, the definitions of optimal answer sets and optimal extended answer sets of  oCAS-program capture formally the semantics of {\sc clingcon-2.1} programs. 

We also note that this observation allows us to utilize pw-MaxSMT solvers (such as, for example,  solver \nuz discussed here) for finding solutions to {\sc clingcon-2.1} programs. Indeed, the essential difference between {\sc clingcon-2.1} programs and  pw-MaxSMT formalism boils down to the first component of the pairs capturing these objects. In one case we deal with CAS programs, in another with SMT formulas. 
Theory behind SMT-based CASP solver {\sc ezsmt}~\citebb{lie17}{shen18a}  relies on the established link between these two. Here, we paved the way at enhancing {\sc ezsmt} with the support for optimization statements.

\paragraph{{\sc clingcon-2.2} programs.}

We now elaborate on  {\sc clingcon-2.2} programs. We start by quoting~\cite{ost12}, who provide examples of optimization statements with constraint variables and informal discussions in order to illustrate their behavior. 
We supplement this quote  with additional comments using square brackets to ease the understanding of the narrative. The {\sc clingcon-2.2} programs allow  expressions of the kind

\begin{verbatim}
$maximize{work(A) : person(A)}.
\end{verbatim}

\noindent
and its developers say that this is an instance of

\vspace{2mm}

\begin{quote}
    a maximize statement over constraint variables. This is also a new feature of clingcon. We maximize the sum over a set
    of variables and/or expressions. In this case, we try to maximize 
    $$\hbox{\tt work(adam) \$+work(smith) \$+ work(lea) \$+ work(john)}.$$
    [During  grounding (see, for instance, an account for grounder {\sc gringo}~\citeb{geb07b}) -- a process that is a common first step in solving ASP programs, where ASP variables are instantiated for suitable object constants -- it was established that ASP variable $A$ may take four values, namely, {\em adam, smith, lea} and {\em john}.]
 ... 
To find a constraint optimal solution, we have to combine the enumeration techniques
of \clasp [by \cite{geb07} -- answer set solver within {\sc clingcon}] with the ones from the CP solver. Therefore, when we first encounter a full propositional assignment, we search for an optimal (w.r.t. to the optimize statement) assignment
of the constraint variables using the search engine of the CP solver. Let us explain this with
the following constraint logic program.

\begin{verbatim}
$domain(1..100).
a :- x $* x $< 25.
$minimize{x}.
\end{verbatim}

\noindent
Assume clasp has computed the full assignment $\{Fx \$* x \$< 25, Fa\}$ [irregular atom named $x \$* x \$< 25$  and regular atom $a$ are assigned value {\em false}; intuitively the mentioned irregular atom is mapped into inequality $x * x < 25$ with constraint variable $x$]. Afterwards,
we search for the constraint optimal solution to the constraint variable $x$, which yields
$\{x \rightarrow 5\}$. Given this optimal assignment, a constraint can be added to the CP solver
that all further solutions shall be below/above this optimum $(x<5)$. This constraint will
now restrict all further solutions to be “better”. We enumerate further solutions, using the
enumeration techniques of \clasp. So the next assignment is $\{Tx \$* x \$< 25, Ta\}$ and
the CP solver finds the optimal constraint variable assignment $\{x \rightarrow 1\}$. Each new solution
restricts the set of further solutions, so our constraint is changed to $(x\$<1)$, which then
allows no further solutions to be found.
\end{quote}

\vspace{3mm}

To  formalize the claims of this quote let us capture {\sc clingcon-2.2} programs as a CAS program $P$ extended with an expression of the form
\beq
\$minimize\{lit_1,\dots,lit_n \},
\eeq{ex:clingcon2.2}
where $lit_i$ ($1\leq i\leq n$) is a constraint variable stemming from irregular atoms of the considered CAS program.
Ew-systems can be used to 
characterize {\sc clingcon-2.2} programs
as follows.
\begin{definition}[{\sc clingcon-2.2} programs and their optimal answer sets]
A {\sc clingcon-2.2} program $P$ (where~$P$ is a CAS program over  $\sigma$, $\C$ and $(\sigma_c,\C)$-denotation) extended with minimization statement~\eqref{ex:clingcon2.2} is ew-system~$(P,W)$,  where $W$ is a single
ew-condition $(T_{\sigma_P,\V_P,\D_P},0;c)$ so that  $c$ is a coefficients function that 
to every constraint variable in $\V_P$ 
assigns~$1$ when it appears in~\eqref{ex:clingcon2.2} and $0$ otherwise. 
Models of this ew-systems are called {\em answer sets} of~$P$, while min-optimal models are called {\em optimal answer sets} of $P$.
\end{definition}

Note how {\sc clingcon-2.1} and {\sc clingon-2.2} programs allow the user to either optimize propositional side of a problem or constraint side of the problem but never both.
We now provide a definition for CAS programs with optimizations restoring to the method adopted earlier in stating the definition for gpw-MaxSMT problem. This  allows us to incorporate quality of ``constraint'' part of the solution into assessment of  overall quality of considered solution. This way quality of propositional and constraint part of solution is taken into account. 
\begin{definition}[CAS program with generalized optimizations; their optimal (extended) answer sets]
A {\em CAS program with generalized optimizations} over vocabulary $\sigma$, class $\C$ of constraints, and $\pair{\sigma_c}{\C}$-denotation  is defined as ew-system  $(H,S)$, where $H$ is a theory in CAS-logic over 
$\sigma$, $\C$, and $\pair{\sigma_c}{\C}$-denotation and $S$ is a collection of ew-conditions, whose theories are in RSMT-logic over 
$\sigma$, $\C$, and $\pair{\sigma_c}{\C}$-denotation.
Min-optimal models and min-optimal extended models of ew-system $(H,S)$ form {\em optimal answer sets and optimal extended answer sets} of  CAS program with generalized optimizations $(H,S)$.
\end{definition}
Note how in this definition ew-conditions are more general than in oCAS-programs or {\sc Clingcon 2.2} programs.

\subsubsection{ {\sc clingcon}-3 style optimizations}

\cite{ban17} introduce optimizations supported within {\sc clingcon}-3. They propose minimize statements for CAS programs that have the form
\beq
\$minimize\{b_1\cdot x_1+c_1@l_1,~\dots,~b_n\cdot x_n+c_n@l_n \},
\eeq{cminimize_statement}
where $b_i$ and $c_i$ are integers, $x_i$ are variables stemming from the constraint part of the program, and  $l_i$ is a level -- positive integer. In addition, for any two expressions $b\cdot x+c@l$  and
$b'\cdot x'+c'@l$ occurring in~\eqref{cminimize_statement} variables $x$ and $x'$ are distinct.
Such minimize statements induce optimal extended answer sets as follows. 
For the remainder of this subsection, let
$P$ be a CAS program  $P'$ extended with minimize statement of the form~\eqref{cminimize_statement}. 
Any (extended) answer set of $P'$ is an {\em (extended) answer set} of $P$.
For a variable assignment~$\nu$ and an integer~$l$, 
 $${\sum_{l}^\nu}=\sum_{b\cdot x+c@l\in \eqref{cminimize_statement}}{b\cdot \nu(x) +c}.$$


\begin{definition}[Optimal answer sets of CAS programs due to~\cite{ban17}]
Let
$(X,\nu)$ and $(X',\nu')$ be extended answer sets of $P$.
Extended answer set $(X',\nu')$ {\em dominates} $(X,\nu)$ if there exists a level $l\in \{l_1,\dots,l_n\}$ ($l_1,\dots,l_n$ are levels occurring in~\eqref{cminimize_statement}) such that
\begin{enumerate}
    \item for any level $l'\in \{l_1,\dots,l_n\}$ that is greater than $l$ the following equality holds
    $$
\displaystyle{\sum_{l'}^\nu} =\displaystyle{\sum_{l'}^{\nu'}},
$$
    \item the following inequality holds for level $l$
    $$
\displaystyle{\sum_{l}^{\nu'}} <\displaystyle{\sum_{l}^{\nu}},
$$
\end{enumerate}
Extended answer set $(X^*,\nu^*)$ of $P$ is optimal if there is no extended answer set $(X',\nu')$ that dominates $(X^*,\nu^*)$.
\end{definition}

For every level $l$ appearing in~\eqref{cminimize_statement}, by 
\begin{itemize}
    \item $w_l$ we denote the sum $\displaystyle{\sum_{b\cdot x +c@l \in  ~\eqref{cminimize_statement}} c}$
    ~~
of constant terms in integer  expressions associated with each level;
    \item $c_l$ we denote coefficient function from $\V_{P'}$ to $\mathbb{Z}$ that maps each variable $x$ occurring in $\V_{P'}$ and in expression 
    $b\cdot x +c@l$ in~\eqref{cminimize_statement} to $b$, while all other variables in $\V_{P'}$ are mapped to $0$.
\end{itemize}
 
As earlier we can identify any CAS program with the CAS-logic module.
For the CAS program~$P'$  extended with the minimize statement~\eqref{cminimize_statement}, we identify~\eqref{cminimize_statement} with the set consisting of the following ew-conditions
\begin{itemize}
    \item for every level $l$ appearing in~\eqref{cminimize_statement}, ew-condition $(T_{\sigma_{P'},\V_{P'},\D_{P'}},w_l@l)$; and
    \item for every level $l$ appearing in~\eqref{cminimize_statement}, ew-condition $(T_{\sigma_{P'},\V_{P'},\D_{P'}},0;c_l@l)$. 
\end{itemize}
Once more we can use ew-systems to provide an alternative definition for CAS programs with minimization statements of the kind introduced in this subsection. 
\begin{proposition}\label{prop:caspminimize}
Let $P$ be a  CAS program $P'$ extended with minimize statements of the form~\eqref{cminimize_statement} over vocabulary 
$\sigma$, class $\C$ of constraints, and $(\sigma_c,\C)$-denotation.
Min-optimal extended models of ew-system~$(P',S)$ ---
where~$S$ is a collection of ew-conditions identified/associated with~\eqref{cminimize_statement} of $P$
---
are optimal answer sets of~$P$.
\end{proposition}

\section{Formal Properties of Ew-systems}\label{sec:formalProperties}
\cite{lie21} stated many interesting formal results for the case of w-systems; 
\cite{lie22} presented proofs for these results.
Many of these results/proofs can be lifted to the case of ew-systems. Here we present a series of formal results about ew-systems. 
 Word {\em Property} denotes the results that follow rather immediately  from the definitions of a  model/optimal (extended) model. 

 \begin{property}\label{prop:one}
Any two ew-systems with the same hard theory have the same models/extended model.
\end{property}
 Due to this property when stating the results for ew-systems that share the same hard theory, we only focus on optimal and min-optimal (extended) models. 
\begin{property}\label{prop:secondempty}
Any model/extended model of ew-system of the form $(\cH,\emptyset)$ is optimal/min-optimal.
\end{property}
\begin{property}\label{prop:remove0}
Optimal/min-optimal models of the following ew-systems coincide
\begin{itemize}
\item ew-system $\cW$ and 
\item ew-system resulting from $\cW$ by dropping all of its w-conditions whose  weight is~$0$.
\end{itemize}
\end{property}

\begin{property}\label{prop:removecew}
Optimal/min-optimal models of the following ew-systems coincide
\begin{itemize}
\item ew-system $\cW$ and 
\item ew-system resulting from $\cW$ by replacing each of its ew-conditions of the form 
$(T,w,\vec{c}@l)$ with ew-condition $(T,w@l)$.
\end{itemize}
\end{property}
This property points at the fact that $\vec{c}$ component of ew-conditions are only relevant when optimality of {\em extended} models is considered.

We now state simple properties that pertain extended models of ew-systems.  
\begin{property}\label{prop:specialForm}
Let $\cW$ be an ew-system, whose ew-conditions have special form $(T,w@l)$.
An extended model $(I,\nu)$ of $\cW$ is 
 optimal/min-optimal if and only if $I$ is an optimal/min-optimal model of $\cW$.
\end{property}

\begin{property}\label{prop:remove00ew}
Optimal/min-optimal extended models of the following ew-systems coincide
\begin{itemize}
\item ew-system $\cW$ and 
\item ew-system resulting from $\cW$ by dropping all of its ew-conditions whose form is  $(T,0@l)$
\end{itemize}
\end{property}

We call an ew-system $\cW$ {\em level-normal}, when
 we can construct the sequence of numbers $1,2,\dots,|\level{\cW}|$ from the elements in $\level{\cW}$.
\cite{lie21} stated propositions in spirit of Propositions~\ref{prop:levelnormal} and~\ref{thm:alloptimal} presented below for the case of w-systems.
Here we lift these results to the case of ew-systems.
\begin{proposition}\label{prop:levelnormal}
 Optimal/min-optimal models/extended models of the following ew-systems coincide
\begin{itemize}
\item ew-system $\cW$ and 
\item the level-normal ew-system constructed from $\cW$ by replacing each level $l_i$ occurring in its ew-conditions with its ascending sequence order number $i$, where we arrange elements in $\level{\cW}$ in a   sequence  in ascending order $l_1,l_2,\dots l_{|\level{\cW}|}$. 
\end{itemize}
\end{proposition}
\begin{proposition}\label{thm:alloptimal}
For an ew-system $\cW=(\cH,\cS)$, if every  level $l\in\level{\cW}$ is such that 
for any distinct models $I$ and $I'$ of $\cW$
 the  equality
 \beq
 \sum_{B\in\cW_{l}}{ \br{I\models B}}= \sum_{B\in\cW_{l}}{ \br{I'\models B}}
 \eeq{eq:eqcond} 
 holds
then optimal/min-optimal models of ew-systems
$\cW$ and  $(\cH,\emptyset)$ coincide. Or, in other words, any model of $\cW$ is also optimal and min-optimal model.
\end{proposition}

The proposition above formulated for  the case of extended models follows.
 
\begin{samepage}
\begin{proposition}\label{thm:alloptimalew}\nopagebreak
For an ew-system $\cW=(\cH,\cS)$, if every  level $l\in\level{\cW}$ is such that 
for any distinct extended models $(I,\nu)$ and $(I',\nu')$ of $\cW$
the  equality
\beq\displaystyle{
\sum_{B\in\cW_l}{ (\br{I\models B} + \br{(I,\nu)\models B})}
=
\sum_{B\in\cW_l}{ (\br{I'\models B} + \br{(I',\nu')\models B})}
}
\eeq{eq:eqcond1} 
holds
then optimal/min-optimal extended models of ew-systems
$\cW$ and  $(\cH,\emptyset)$ coincide. Or, in other words, any extended model of $\cW$ is also optimal and min-optimal model.
\end{proposition}
\end{samepage}

Let $\cW=(\cH,\cS)$ be  an ew-system. For a set $S$ of ew-conditions, 
by $\less{\cW}{S}$ we denote the ew-system $(\cH,\cS\setminus S)$.

\begin{proposition}\label{thm:samewcond}
For a ew-system $\cW=(\cH,\cS)$, if there is a set $S\subseteq \cS$ of ew-conditions all sharing the same level  $l$ 
such that 
for any distinct $\prec{l}$-optimal/min-optimal models $I$ and $I'$ of $\cW$
(or any distinct  models $I$ and $I'$ of $\cW$ in case $\prec{l}$ is undefined)
 the equality~\eqref{eq:eqcond}, where $\cW_{l}$ is replaced by $S$, holds 
then $\cW$ has the same  optimal/min-optimal models as  $\less{\cW}{S}$. 
\end{proposition}
We can formulate a similar claim for the case of extended models.


\begin{proposition}\label{thm:samewcondtwo}
For an ew-system $\cW=(\cH,\cS)$, if there is a set $S\subseteq \cS$ of ew-conditions all sharing the same level  such that 
for any distinct  $\prec{l}$-optimal/min-optimal  extended models $(I,\nu)$ and $(I',\nu')$ of $\cW$ 
(or any distinct  extended models $(I,\nu)$ and $(I',\nu')$ of $\cW$ in case $\prec{l}$ is undefined)
the equality~\eqref{eq:eqcond1}, where $\cW_{l}$ is replaced by $S$, holds
then $\cW$ has the same  optimal/min-optimal extended models as~$\less{\cW}{S}$. 
\end{proposition}

 For a coefficients function~$\vec{c}$ mapping variables  into reals by~$\vec{c}^{-1}$ we denote a function 
 on the same set of variables as~$\vec{c}$
 defined 
  as follows $$\vec{c}^{-1}(x)=-1\cdot \vec{c}(x).$$
 For ew-condition $(T,w;\vec{c}@l)$, we define two mappings into related ew-conditions
 $$
 \ba{l}
 \signo{(T,w;\vec{c}@l)}=(T,-1\cdot w;\vec{c}@l),\\
 \signoo{(T,w;\vec{c}@l)}=(T,-1\cdot w;\vec{c}^{-1}@l).\\
 \ea
 $$
For ew-system $(\cH,\{B_1,\dots, B_n\})$, we define two mappings into related ew-systems using concepts above 
$$
\ba{l}
\signo{(\cH,\{B_1,\dots, B_n\})}=(\cH,\{\signo{B_1},\dots, \signo{B_n}\})\\
\signoo{(\cH,\{B_1,\dots, B_n\})}=(\cH,\{\signoo{B_1},\dots, \signoo{B_n}\})\\
\ea
$$
With this newly introduced notation we can now claim the relation between optimal and min-optimal (extended) models of ew-systems.

\begin{proposition}\label{prop:relatives}
For an  ew-system $\cW$,
the optimal models (min-optimal models) of $\cW$ coincide with the min-optimal models (optimal models) of $\signo{\cW}$. 
\end{proposition}
\begin{proposition}\label{prop:relatives2}
For an  ew-system $\cW$,
the extended optimal models (extended min-optimal models) of $\cW$ coincide with the extended min-optimal models (extended optimal models) of $\signoo{\cW}$. 
\end{proposition}





\paragraph{Eliminating Negative (or Positive) Weights}

We call e-logics $\cL$ and $\cL'$ {\em compatible} when 
their vocabularies, domains, and variables coincide, i.e.,  \hbox{$\sigma_\cL=\sigma_{\cL'}$, $\D_{\cL}=\D_{\cL'}$, and $\V_{\cL}=\V_{\cL'}$}. 
Let $\cL$ and $\cL'$ be compatible logics, and
$T$ and $T'$ be theories in these logics, respectively.
We call a theory $T$ (and a w-condition $(T,w;\vec{c}@l)$) {\em equivalent to}
 a theory $T'$ (and a w-condition $(T',w;\vec{c}@l)$, respectively), when $sem(T)= sem(T')$.

\begin{property}\label{propo:equivalent}
Models and optimal/min-optimal models/extended models  of  ew-systems 
$$(\{T_1,\dots,T_n\},\{B_1,\dots,B_m\}) \hbox{ and } (\{T'_1,\dots,T'_n\},\{B'_1,\dots,B'_m\})$$ 
coincide when
(i) $T_i$ and $T'_i$ ($1\leq i\leq n$) are equivalent theories, and 
(ii) $B_i$ and $B'_i$ ($1\leq i\leq m$)  are equivalent ew-conditions.
\end{property}

For a theory $T$ of e-logic $\cL$, we call a theory $\overline{T}$ in e-logic~$\cL'$, compatible to~$\cL$, 
{\em complementary} when 
(i) $sem(T)\cap sem(\overline{T})=\emptyset$, and
(ii) $sem(T)\cup sem(\overline{T}) = \Int(\sigma_{\cL},\V_{\cL},\D_{\cL})$.

Let $(T,w;\vec{c}@l)$ be an ew-condition; consider the following definitions:
$$
\ba{l}
\signp{(T,w;\vec{c}@l)}=\begin{cases}
  (T,w;\vec{c}@l)&\hbox{when $w\geq 0$, otherwise }  \\
  (\overline{T},-1\cdot w;\vec{c}@l)   \\
\end{cases}\\
~\\
\signm{(T,w;\vec{c}@l)}=\begin{cases}
  (T,w;\vec{c}@l)&\hbox{when $w\leq 0$, otherwise }  \\
  (\overline{T},-1\cdot w;\vec{c}@l)   \\
  \end{cases}\\
~\\
\signpp{(T,w;\vec{c}@l)}=\begin{cases}
  (T,w;\vec{c}@l)&\hbox{when $w\geq 0$, otherwise }  \\
  (\overline{T},-1\cdot w;\vec{c}^{-1}@l)   \\
\end{cases}\\
~\\
\signmm{(T,w;\vec{c}@l)}=\begin{cases}
  (T,w;\vec{c}@l)&\hbox{when $w\leq 0$, otherwise }  \\
  (\overline{T},-1\cdot w;\vec{c}^{-1}@l)   \\
\end{cases}
\ea
$$
where $\overline{T}$ denotes some theory complement to $T$.

For an ew-system $(\cH,\{B_1,\dots,B_m\})$, 
we define 
\beq
\ba{l}
\signp{(\cH,\{B_1,\dots, B_n\})}=(\cH,\{\signp{B_1},\dots, \signp{B_n}\}),\\
\signm{(\cH,\{B_1,\dots, B_n\})}=(\cH,\{\signm{B_1},\dots, \signm{B_n}\}),\\
\signpp{(\cH,\{B_1,\dots, B_n\})}=(\cH,\{\signpp{B_1},\dots, \signpp{B_n}\}),\\
\signmm{(\cH,\{B_1,\dots, B_n\})}=(\cH,\{\signmm{B_1},\dots, \signmm{B_n}\}).\\
\ea
\eeq{eq:plusminusplusminus}

\begin{proposition}\label{prop:signpsignm}
Optimal/min-optimal models of  ew-systems  $\cW$, $\signp{\cW}$,   $\signm{\cW}$, $\signpp{\cW}$,   $\signmm{\cW}$  coincide.
\end{proposition}
This proposition  can be seen as an immediate consequence of the following  result and Property~\ref{prop:removecew}:
\begin{proposition}\label{prop:signpsignm2}
Optimal/min-optimal models of  ew-systems  
\begin{itemize}
    \item 
$(\cH,\{(T,w@l)\}\cup\cS)$
and 
\item $(\cH,\{(\overline{T},-1\cdot w@l)\}\cup\cS)$ 
\end{itemize}
coincide.
\end{proposition}

\begin{proposition}\label{prop:signpsignmsignpsignm}
 Optimal/min-optimal extended models of  ew-systems  $\cW$, $\signpp{\cW}$,   $\signmm{\cW}$  coincide.
\end{proposition}
This proposition  can be seen as an immediate consequence of the following  result:
\begin{proposition}\label{prop:signsignpsignm2}
Optimal/min-optimal extended models of  ew-systems 
\begin{itemize}
    \item $(\cH,\{(T,w;\vec{c}@l)\}\cup\cS)$ and
\item $(\cH,\{(\overline{T},-1\cdot w;\vec{c}^{-1}@l)\}\cup\cS)$
\end{itemize} 
coincide.
\end{proposition}

The notation and results of the section on {\em Eliminating Levels} by~\cite{lie21,lie22} can be lifted to the case of ew-systems  and their optimal/min-optimal models in a straightforward manner (similar as the results for eliminating negative/positive weights are lifted here). Hence, we omit the review of these results. Yet, the ``cost'' expression associated with the second member of extended model is more complex so that the role of the  levels seem to go beyond syntactic sugar when extended models are considered.

\section{Proofs}
Many of the formal properties of ew-systems presented in Section~\ref{sec:formalProperties} echo similar results for w-systems --- a precursor of ew-systems introduced by~\cite{lie21}. \cite{lie22} presented proofs for these results for w-systems. The logic and structure of proofs for the case of ew-systems follows the proofs for the case of w-systems. In this section, we often  remark on the connection to the proofs by \cite{lie22} and point at any details worth noting.

Just as in case for w-systems, we focus on results for optimal (extended) models only, as the arguments for min-optimal (extended) models follow the same lines.
Given recursive Definitions~\ref{def:oemewsys} and~\ref{def:oemews} of $l$-(min)-optimal (extended) models,  inductive argument is a common technique in proof construction about properties of such models. In particular,  the {\em induction on levels of a considered  ew-system $\cW$}, where we assume elements in $\level{\cW}$ to be arranged in the descending order $m_1,\dots m_n$ ($n=|\level{\cW}|$); so that the base case is illustrated for the greatest level $m_1$, whereas inductive hypothesis is assumed for level $m_i$ and then illustrated to hold for level $m_{i+1}$. Note how, 
$\prec{m_{i+1}}=m_{i}$.

\begin{proof}[Proof of Proposition~\ref{prop:eqdefs}]%
The statement of Proposition~\ref{prop:eqdefs} echos the statement of Proposition~1 by \cite{lie22} for the case of w-systems.
The proofs of the two claims (i) Definitions~\ref{def:oemewsys} and~\ref{def:optimalmodelewsysASP} are equivalent; and (ii) 
Definitions~\ref{def:oemews} and~\ref{def:optimalemodelewsysASP} are equivalent
follow the lines of the proof provided for Proposition~1 by \cite{lie22}. In fact, for the claim (i) we can repeat the proof practically verbatim modulo the understanding that in place of w-system $\cW$ we consider ew-system  $\cW$ as well as in place of w-conditions of $\cW$ we consider ew-conditions of $\cW$.
For the claim (ii), in addition to  the proof of Proposition~1 by \cite{lie22} will have to refer to extended interpretations in place of interpretations and to summations of the form
$$\sum_{B\in\cW_l}{ (\br{I\models B} + \br{(I,\nu)\models B})}$$
in place of summations of the form
$$\sum_{B\in\cW_l}{ (\br{I\models B}}.$$
Claims of Lemmas~1,~2, and~3  by \cite{lie22} hold for the case of ew-systems and their $l$-optimal models and $l$-optimal extended models.
\end{proof}

\begin{proof}[Proof of Proposition~\ref{prop:wmaxsmt}]
Let $S$ be a weighted MaxSMT problem and $\cW_S=(T_{\sigma_S,\V_S,\D_S},S)$ be a respective ew-system ---
where each element in $S$ is understood as an  ew-condition, whose theory is in SMT-logic.
Consider an arbitrary interpretation $I^*\in Int(\sigma_S)$. We show that 
$I^*$ is a solution to weighted MaxSMT problem $S$ if and only if 
$I^*$ is an optimal models of~$\cW_S$.
Per Definition~\ref{def:solwmsmt}, interpretation $I^*$  is a  solution to   $S$ if and only if 
$$
\ba{l}
\displaystyle{I^*=\argmax_{I}{\sum_{(\cF,w)\in S}{w\cdot \br{I\models \cF}}}},
\ea
$$
where $I$ ranges over all interpretations in $Int(\sigma_S)$.
By the definition of $T_{\sigma_S,\V_S,\D_S}$,
 any interpretation in $Int(\sigma_S)$ is its  model.
Per Definition~\ref{def:oemewsys}, interpretation $I^*$  is an optimal model of  $\cW_S$ if and only if 
$$
I^*=\displaystyle{ arg\max_{I} {\sum_{(\cF,w)\in S}{ \br{I\models (\cF,w)}}}},
$$
where $I$ ranges over all interpretations in $Int(\sigma_S)$.

Taking definitions~\eqref{eq:isat} and~\eqref{eq:smteq} into account and an understanding that   each element in $S$ can be seen as an  ew-condition, whose theory is in SMT-logic, we conclude that for any interpretation $I$ in $Int(\sigma_S)$
 $$
 {\sum_{(\cF,w)\in S}{w\cdot \br{I\models \cF}}}=
 {\sum_{(\cF,w)\in S}{ \br{I\models (\cF,w)}}}.
$$
\end{proof}

\begin{proof}[Proof of Proposition~\ref{prop:caspminimize}]
Let $P$ be a  CAS program $P'$ extended with minimize statements of the form~\eqref{cminimize_statement} over vocabulary~$\sigma$, class $\C$ of constraints, and $(\sigma_c,\C)$-denotation.
Consider an extended answer set $(X^*,\nu^*)$ of $P$. We show that 
$(X^*,\nu^*)$ is optimal extended answer set of $P$ if and only if 
$(X^*,\nu^*)$ is
min-optimal extended model of ew-system~$(P',S)$ ---
where~$S$ is a collection of ew-conditions identified/associated with~\eqref{cminimize_statement} of $P$.

It is easy to see that any extended answer set of $P$
 is an extended model of $(P',S)$. Thus, the proof focuses on optimality condition.

Per Definition~\ref{def:oascc3},  $(X^*,\nu^*)$  is an optimal extended answer set if and only if there is no extended answer set 
 $(X',\nu')$ that dominates  $(X^*,\nu^*)$. In other words, any answer set $(X',\nu')$ is such that for every level $l$ occurring in~\eqref{cminimize_statement} either 
 \begin{enumerate}
    \item there exists a level  $l'$  occurring in~\eqref{cminimize_statement} that is greater than $l$ and the following inequality holds
    $$
\displaystyle{\sum_{l'}^{\nu^*}} \neq \displaystyle{\sum_{l'}^{\nu'}},
$$ or
    \item the following inequality holds for level $l$
    $$
\displaystyle{\sum_{l}^{\nu'}} \geq \displaystyle{\sum_{l}^{\nu^*}},
$$
\end{enumerate}

Per Definition~\ref{def:optimalemodelewsysASP}, $(X^*,\nu^*)$ is
min-optimal extended model of ew-system~$(P',S)$ if and only if there is no extended model~$(I',\nu')$ of~$(P',S)$ that min-dominates  $(X^*,\nu^*)$.   In other words, any extended model $(X',\nu')$  is such that for every level $l\in \level{(P',S)}$ 
 either 
 \begin{enumerate}
    \item there exists a level  $l'\in \level{(P',S)}$ that is greater than $l$ and the following inequality holds
$$
\displaystyle{ 
{\sum_{B\in (P',S)_{l'}}{ (\br{X^*\models B} + \br{(X^*,\nu^*)\models B})}}
}
\neq
\displaystyle{ 
{\sum_{B\in(P',S)_{l'}}{ (\br{X'\models B} + \br{(X',\nu')\models B})}}
}
$$
 or
    \item the following inequality holds for level $l$
    $$
\displaystyle{ 
{\sum_{B\in(P',S)_{l}}{ (\br{X'\models B} + \br{(X',\nu')\models B})}}
}
\geq 
\displaystyle{ 
{\sum_{B\in(P',S)_{l}}{ (\br{X^*\models B} + \br{(X^*,\nu^*)\models B})}}
}
$$ 
\end{enumerate}

We first observe that by the construction of $S$ any level $l$ occurs in~\eqref{cminimize_statement} if and only if \hbox{$l\in \level{(P',S)}$}.
Recall that  any extended answer set of $P$ is an extended model of $(P',S)$. It is now sufficient to show that given any extended model $(X,\nu)$ of $(P',S)$, the following equality holds for arbitrary level $l\in \level{(P',S)}$:
\beq
\displaystyle{\sum_{l}^{\nu}}=
\displaystyle{ 
{\sum_{B\in (P',S)_{l}}{ (\br{X\models B} + \br{(X,\nu)\models B})}}
}.
\eeq{eq:goal}
Indeed, per definition of ${\sum_{l}^\nu}$
 \beq
 {\sum_{l}^\nu}=\sum_{b\cdot x+c@l \hbox{ in } \eqref{cminimize_statement}}{(b\cdot \nu(x) +c)}= 
 \sum_{b\cdot x+c@l \hbox{ in } \eqref{cminimize_statement}}{b\cdot \nu(x)}+\sum_{b\cdot x+c@l \hbox{ in } \eqref{cminimize_statement}}{c}.
 \eeq{eq:subgoal1}

Per $S$ construction, $(P',S)_{l}$ consists of two ew-conditions
\begin{itemize}
    \item  $(T_{\sigma_{P'},\V_{P'},\D_{P'}},w_l@l)$; and
    \item  $(T_{\sigma_{P'},\V_{P'},\D_{P'}},0;c_l@l)$. 
\end{itemize}

It is easy to see that $$[(X,\nu)\models (T_{\sigma_{P'},\V_{P'},\D_{P'}},w_l@l)]= 0, $$
$$[X\models  (T_{\sigma_{P'},\V_{P'},\D_{P'}},0;c_l@l)]= 0. $$
Per definitions of $w_l$ and $c_l$,
we derive that
$$[X\models (T_{\sigma_{P'},\V_{P'},\D_{P'}},w_l@l)]=w_l= \sum_{b\cdot x +c@l \hbox{ in } ~\eqref{cminimize_statement}} {c}, $$
and
$$[(X,\nu)\models (T_{\sigma_{P'},\V_{P'},\D_{P'}},0;c_l@l)]= \sum_{x\in \V_{P'}}{\nu(x)\cdot c_l(x)} = \sum_{b\cdot x+c@l \hbox{ in } \eqref{cminimize_statement}}{b\cdot \nu(x)}.$$
Consequently,
\beq
{\sum_{B\in (P',S)_{l}}{ (\br{X\models B} + \br{(X,\nu)\models B})}}= 
 \sum_{b\cdot x+c@l \hbox{ in } \eqref{cminimize_statement}}{b\cdot \nu(x)}+\sum_{b\cdot x+c@l \hbox{ in } \eqref{cminimize_statement}}{c}.
\eeq{eq:subgoal2}
Equality~\eqref{eq:goal} follows immediately from~\eqref{eq:subgoal1} and~\eqref{eq:subgoal2}. 
\end{proof}

Proof of Proposition~\ref{prop:levelnormal} follows from
the fact that the numeric value of any level itself is inessential in the key computations associated with establishing optimal (extended) models. Rather, the order of levels with respect to greater relation  matters (recall the definition of $\prec{(\cdot)}$ operation). It is easy to see that changing levels of the w-conditions using the procedure described in this proposition preserves original order of the levels with respect to greater relation.

Propositions~\ref{thm:alloptimal} and~\ref{thm:alloptimalew} follow immediately from Propositions~\ref{thm:samewcond} and~\ref{thm:samewcondtwo}, respectively. 
The statement of Proposition~\ref{thm:samewcond} echos the statement of Proposition~8 by \cite{lie22} for the case of w-systems.
The statement of Proposition~\ref{thm:samewcondtwo} lifts the statement of Proposition~\ref{thm:samewcond} from the case of models to the case of extended models.
The proofs of Propositions~\ref{thm:samewcond} and~\ref{thm:samewcondtwo}
follow the lines of the proof provided for Proposition~8 by \cite{lie22} modulo similar provisions as pointed at in the beginning of this section in {\em Proof of Proposition~\ref{prop:eqdefs} (sketch)}.

The statement and proof of Proposition~\ref{prop:relatives} echos the statement and proof of Proposition~9 by \cite{lie22} for the case of w-systems. 
The structure of the following proof is in spirit of the proof  of Proposition~9 by \cite{lie22} and, yet, we state some of its details here as  mapping~$\signoo{(\cdot)}$ is unique to this work.

\begin{proof}[Proof of Proposition~\ref{prop:relatives2}]
To show that 
the extended optimal models  of $\cW$ coincide with the min-optimal extended models  of $\signoo{\cW}$, it is sufficient to show that for any level in $\level{\cW}$, $l$-optimal extended models of $\cW$ coincide with $l$-min-optimal extended models of  $\signoo{\cW}$.
We first note that extended models of $\cW$ and $\signoo{\cW}$ coincide.
By induction on levels of  $\cW$.

Base case. $l$ is the greatest level.
An extended model $(I^*,\nu^*)$ of $\cW$ is {\em $l$-optimal} if and only if
$(I^*,\nu^*)$ satisfies equation~\eqref{eq:condeqlmin2},
where
$(I,\nu)$ ranges over extended models of $\cW$.
It is easy to see that we can rewrite this equation as 
$$
(I^*,\nu^*)=
\displaystyle{ arg\min_{(I,\nu)} {
\sum_{B\in\cW_l}{ \Big(  -1\cdot \br{I\models B} + (-1 \cdot \br{(I,\nu)\models B})\Big)}}.
}
$$
It immediately follows from the construction of $\signoo{\cW}$ that this equation can be rewritten as
$$
(I^*,\nu^*)=
\displaystyle{ arg\min_{(I,\nu)} {
\sum_{B\in\signoo{\cW}_l}{ (  \br{I\models B} + \br{(I,\nu)\models B})}}.
}
$$
Thus $(I^*,\nu^*)$ is an $l$-min-optimal extended model of $\signoo{\cW}$ as the equation above is exactly the one from the definition of $l$-min-optimal extended models of  $\signoo{\cW}$; plus recall that extended models of $\cW$ and $\signoo{\cW}$ coincide.

Inductive case argument follows similar lines.
\end{proof} 

Propositions~\ref{prop:signpsignm} and~\ref{prop:signpsignmsignpsignm}  follow from Propositions~\ref{prop:signpsignm2} and~\ref{prop:signsignpsignm2}, respectively. 
Proofs of Propositions~\ref{prop:signpsignm2} and~\ref{prop:signsignpsignm2} follow the lines of proof of Proposition 11 stated by \cite{lie22}.

\section{Conclusions and Acknowledgments}

We trust that the proposed unifying framework of ew-systems will allow  developers of distinct automated paradigms to better grasp similarities and differences of the kind of optimization criteria their paradigms support. In practice, translational approaches are popular in devising solvers. These approaches rely on the established relationships between automated reasoning paradigms. In particular, rather than devising a unique search algorithm for a paradigm of interest, researchers propose a translation from this ``source'' paradigm to another ``target'' framework. As a result  solvers for the target framework are used to find solutions for a problem encoded in the source paradigm. This work is a stepping stone towards extending these translational approaches with the support for optimization statements. 
 We proposed the extension of abstract modular systems to extended weighted systems in a way that modern approaches to optimizations stemming from a variety of different logic based formalisms can be studied in unified terminological ways so that their differences and similarities become clear not only on intuitive but also formal level. These ew-systems allowed us to provide generalizations for the family of MaxSMT problems that incorporate optimizations over theory/constraint elements of these problems in addition to their propositional side. The framework also provides immediate support for the concept of levels of optimization criteria. We also provided formal semantics for two variants of {\sc clingcon-2} programs that mimic the behavior of their informal descriptions in the literature.
 We trust that establishing clear link between optimization statements, criteria, and solving in distinct automated reasoning subfields is a truly fruitful endeavor allowing a streamlined cross-fertilization between the fields. 
 The {\sc ezsmt}~\citeb{shen18a} system is a translational constraint answer set solver that translates its programs into satisfiability modulo theories  formulas. We trust that results obtained here lay the groundwork for extending a ``translational'' solver {\sc ezsmt} with the support for optimization statements.
 
 {The work was partially supported by NSF grant 1707371.} We are grateful to anonymous reviewers for valuable comments on this paper.
 
Competing interests: The author(s) declare none.

\bibliographystyle{tlplike}
\bibliography{abstractmods-bib}

\label{lastpage}
\end{document}